\documentclass{article}

\usepackage[preprint]{arxiv}

\usepackage[utf8]{inputenc}
\usepackage[T1]{fontenc}
\usepackage{hyperref}
\usepackage{url}
\usepackage{booktabs}
\usepackage{graphicx}
\usepackage{amsfonts}
\usepackage{amssymb}
\usepackage{mathtools}
\usepackage{amsthm}
\usepackage{nicefrac}
\usepackage{microtype}
\usepackage[table]{xcolor}
\usepackage{algorithm}
\usepackage{algpseudocode}

\title{Correcting Stochastic Update Bias in Preconditioned Language Model Optimizers}

\author{%
  \textbf{Nikhil Nayak, Julia White, Urchade Zaratiana, Kelton Zhang,} \\
  \textbf{Henrijs Princis, Dhruv Atreja, Henry Fawcett, Matthew Thomas,} \\
  \textbf{George Hurn-Maloney, Ash Lewis} \\~\\
  \textbf{Fastino Labs} \\~\\
  Correspondence: \texttt{nayak.nikhil2608@gmail.com} \\
  Code: \url{https://github.com/fastino-ai/preconditioner-bias-correction}
}

\newcommand{\E}{\mathbb{E}}
\newcommand{\D}{\mathcal{D}}
\newcommand{\B}{\mathcal{B}}

\newcommand{\Var}{\mathrm{Var}}
\definecolor{headergray}{gray}{0.90}
\definecolor{sectiongray}{gray}{0.95}
\newcommand{\best}[1]{\textbf{#1}}

\mathtoolsset{showonlyrefs}

\newtheorem{assumption}{Assumption}
\newtheorem{theorem}{Theorem}
\newtheorem{corollary}{Corollary}

\begin{document}

\maketitle

\begin{abstract}
Preconditioned optimizers are central to language model training, but their stochastic update rules are usually treated as direct approximations to population preconditioned descent. We show that this view misses two finite-sample biases. First, the gradient and preconditioner are typically estimated from the same minibatch, introducing gradient--preconditioner coupling bias. Second, even when the preconditioner estimate is unbiased, its inverse or inverse-root is generally biased because inversion is nonlinear. We propose a single-batch bias-correction framework that addresses both effects: cross-fitted preconditioning estimates the numerator and preconditioner from independent microbatch groups, while variance-corrected inversion uses microbatch variability to subtract the leading delta-method bias term. The framework applies to diagonal moment, diagonal curvature, and matrix preconditioning methods, instantiated in AdamW, Sophia, and Shampoo. Bias correction reduces held-out pretraining loss on Qwen2.5-0.5B by $0.15$, $0.07$, and $0.11$ nats, respectively; the effects on mixed-quality pretraining and downstream instruction tuning are consistently neutral-to-positive. Together, these results establish bias correction as a practical mechanism for reducing finite-sample update bias and improving the performance of preconditioned optimizers.
\end{abstract}

\section{Introduction}
\label{sec:introduction}

Large language models are usually trained with first-order stochastic optimizers such as SGD, Adam, AdamW, and their matrix- or curvature-aware variants. These methods work well in practice, but their update rules are usually treated as deterministic transformations of stochastic quantities. In particular, adaptive optimizers estimate both a descent direction and a geometry term from minibatches. The geometry term may be a scalar learning-rate estimate, a diagonal second-moment estimate, a curvature proxy, or a structured matrix preconditioner.

This paper focuses on a statistical question that is usually implicit in optimizer design. For the population objective
\begin{equation}
    f(\theta) = \E_{z \sim \D}[\ell(\theta; z)]
\end{equation}
and a population preconditioner $P(\theta)$, the ideal preconditioned update at parameter value $\theta$ is
\begin{equation}
    u^\star(\theta) = P(\theta)^{-1} \nabla f(\theta),
    \label{eq:ideal-update}
\end{equation}
whereas in practice we only observe minibatch estimates $\hat g$ and $\hat P$ and apply the standard adaptive update
\begin{equation}
    \hat u = \hat P^{-1} \hat g.
    \label{eq:standard-update}
\end{equation}
Even when $\hat g$ is an unbiased gradient estimator and $\hat P$ is an unbiased preconditioner estimator, the product in Equation~\eqref{eq:standard-update} is not generally an unbiased estimator of Equation~\eqref{eq:ideal-update}.

There are two distinct sources of mismatch. First, if $\hat P$ and $\hat g$ are estimated from the same random minibatch, the two terms are statistically coupled. Second, even if $\hat P$ is unbiased for $P$, the inverse $\hat P^{-1}$ is generally biased for $P^{-1}$ because inversion is nonlinear. These effects are not implementation mistakes. They arise from using finite stochastic estimates inside nonlinear adaptive update rules.

The goal of this paper is to make these two biases explicit and propose a practical correction framework. The proposed method is single-step and single-batch: it does not require revisiting the same parameter value, which is unrealistic in neural network training. At a fixed $\theta_t$, the current training batch is split into microbatch groups. One group estimates the gradient, another estimates the preconditioner, and microbatch variation is used to correct the inverse-preconditioner estimate before the update is applied.

We evaluate the framework on three optimizer families. AdamW represents diagonal moment-based adaptive optimization. Sophia represents diagonal curvature-based or lightweight second-order optimization. Shampoo represents matrix-preconditioned optimization. These three cases test whether the proposed bias correction is useful across increasingly structured preconditioners: diagonal moments, diagonal curvature, and matrix geometry.

\paragraph{Contributions.}
This paper makes three contributions.
\begin{enumerate}
    \item We give a statistical decomposition of the stochastic preconditioned update into two distinct finite-sample biases: gradient--preconditioner coupling, and the nonlinearity of inversion.
    \item We propose a single-batch correction framework that combines cross-fitted preconditioning with a delta-method inverse correction estimated from microbatch variability, and we analyze its bias-reduction and convergence properties.
    \item We instantiate the framework across diagonal moment, diagonal curvature, and matrix preconditioning (AdamW, Sophia, Shampoo) and evaluate it on pretraining and instruction tuning; the strongest empirical results are $0.1489$-, $0.0701$-, and $0.1103$-nat reductions in held-out pretraining cross-entropy for AdamW, Sophia, and Shampoo.
\end{enumerate}

\section{Related work}
\label{sec:related-work}

\textbf{Momentum, acceleration, and convergence guarantees.}
The study of faster first-order optimization predates modern adaptive methods. Robbins and Monro introduced stochastic approximation as a framework for iterative optimization from noisy observations \citep{robbins1951stochastic}. Polyak's heavy-ball method showed how momentum can accelerate deterministic iterative methods \citep{polyak1964some}, while Nesterov's accelerated method achieved the optimal $O(1/k^2)$ rate for smooth convex optimization with first-order information \citep{nesterov1983method}. Polyak--Ruppert averaging later showed that averaging stochastic approximation iterates can recover optimal asymptotic behavior under classical conditions \citep{polyak1992acceleration}. For nonconvex stochastic optimization, guarantees typically shift from convergence to a global minimum to convergence toward approximate stationarity; for example, \citet{ghadimi2013stochastic} give complexity guarantees for stochastic first-order methods in smooth nonconvex problems. In deep learning, momentum remains central: \citet{sutskever2013importance} showed that carefully scheduled momentum and initialization allow SGD with momentum to train deep and recurrent networks effectively, and Lookahead wraps an inner optimizer with slow/fast weights to improve stability and reduce variance \citep{zhang2019lookahead}. These works improve convergence rates, stability, or asymptotic behavior of the optimization trajectory. Our work addresses a different but compatible question: whether the stochastic preconditioned update used by such optimizers is itself a biased estimator of the corresponding population preconditioned direction.

\textbf{Adaptive diagonal optimizers.}
AdaGrad introduced coordinate-wise adaptive scaling from accumulated squared gradients, with both diagonal and full-matrix viewpoints \citep{duchi2011adaptive}. Adam made this form of adaptivity the default in deep learning by combining exponential moving averages of first and second moments with initialization bias corrections for those moving averages \citep{kingma2015adam}. AdamW decoupled weight decay from the adaptive gradient update, which is now standard in language model training \citep{loshchilov2019decoupled}. Subsequent work identified convergence pathologies in Adam-like exponential moving average schemes and proposed AMSGrad-style long-memory fixes \citep{reddi2018convergence}. Other widely used variants target optimizer memory or large-batch stability: Adafactor factorizes second-moment accumulators to reduce memory from matrix size to row/column size \citep{shazeer2018adafactor}, while LAMB adds layerwise trust ratios for large-batch BERT training \citep{you2020large}. These methods design the form of the preconditioner, its memory footprint, or its large-batch scaling. Our work is orthogonal: given a stochastic preconditioner estimate, we study the bias introduced by multiplying it with a stochastic gradient and by applying a nonlinear inverse. This is also distinct from Adam's standard bias correction, which corrects the initialization bias of exponential moving averages, not the finite-sample bias of the inverse preconditioner.

\textbf{Curvature and matrix preconditioning.}
Natural gradient methods use the Fisher geometry as the local metric for optimization \citep{amari1998natural}. K-FAC makes this idea practical for neural networks by approximating layerwise Fisher blocks with Kronecker factors \citep{martens2015optimizing}. Shampoo instead maintains structured second-moment matrices for tensor dimensions and applies inverse-root preconditioners \citep{gupta2018shampoo}; scalable implementations add numerical and systems refinements that make Shampoo practical for deep networks and language models \citep{anil2021scalable}. SOAP further combines Shampoo's eigenspace with Adam-like second-moment tracking to reduce the degradation from infrequent eigendecompositions \citep{vyas2024soap}. In the diagonal-curvature direction, Sophia estimates a diagonal Hessian or Gauss--Newton--Bartlett-style curvature proxy and clips the resulting curvature-scaled update for language model pretraining \citep{liu2024sophia}. These methods estimate richer geometry than AdamW. Our contribution is not a new curvature approximation; it is a statistical correction layer for the stochastic inverse or inverse-root already used by such optimizers.

\textbf{Variance reduction, stability, and objective modification.}
Classical variance-reduced stochastic methods such as SVRG and SAGA construct lower-variance gradient estimators by reusing reference or stored gradients \citep{johnson2013accelerating,defazio2014saga}. These methods reduce gradient noise but do not address bias from inverting a noisy preconditioner. Sharpness-Aware Minimization modifies the training objective to favor parameters whose neighborhoods have uniformly low loss \citep{foret2021sharpness}. This is another complementary axis: SAM changes the objective being optimized, whereas our method changes the statistical estimator of a preconditioned update for a fixed objective. Clipping and trust-ratio mechanisms in optimizers such as Adafactor, LAMB, and Sophia can stabilize large or poorly scaled updates, but they do not explicitly remove the leading inverse-estimator bias; in our experiments, such stabilization can either help preserve useful same-batch normalization or mask the effect of denominator correction.

\textbf{Statistical bias correction and sample splitting.}
The jackknife was developed as a general bias-reduction technique for estimators \citep{quenouille1956bias,efron1982jackknife}, and the delta method gives the standard second-order expansion used to approximate the bias of smooth nonlinear transformations \citep{vandervaart1998asymptotic}. Cross-fitting and sample splitting are also central in double/debiased machine learning, where they reduce overfitting and nuisance-estimation bias in semiparametric inference \citep{chernozhukov2018double}. We use these statistical tools in a different setting: a single optimizer step at a fixed parameter vector. The numerator and denominator of the update are estimated from independent microbatch groups to remove gradient--preconditioner coupling, and microbatch variability is used to subtract the leading bias introduced by inverse or inverse-root maps. To our knowledge, prior optimizer work has not turned these classical bias-correction ideas into a practical correction for AdamW-, Sophia-, and Shampoo-style language model optimizers.

\section{Background and problem setup}
\label{sec:background}

\subsection{Stochastic first-order optimization}

For the objective $f(\theta) = \E_{z \sim \D}[\ell(\theta; z)]$, plain SGD updates parameters as $\theta_{t+1} = \theta_t - \eta \hat g_\B$ using the unbiased minibatch gradient
\begin{equation}
    \hat g_\B = \frac{1}{|\B|}\sum_{z_i \in \B} \nabla_\theta \ell(\theta; z_i),
    \qquad
    \E[\hat g_\B] = \nabla f(\theta).
\end{equation}
Here the preconditioner is a deterministic identity, so neither bias studied in this paper arises: SGD has gradient variance, but no gradient--preconditioner coupling bias and no inverse-preconditioner bias.

Both biases appear once the update contains a stochastic, data-dependent multiplier---as in adaptive SGD variants, stochastic line-search methods, gradient clipping, normalized gradients, Adam-style methods, or matrix-preconditioned optimizers.
All of these can be written abstractly as
\begin{equation}
    \theta_{t+1} = \theta_t - \eta \hat P_t^{-1} \hat g_t,
\end{equation}
where $\hat P_t$ may be diagonal, block-diagonal, low-rank, or matrix-valued. 
The statistical question, which standard convergence analyses do not address, is whether $\hat P_t^{-1}\hat g_t$ is a clean estimator of the ideal population update.

\subsection{Adaptive and preconditioned optimizers}

Adam-style optimizers use coordinate-wise second moments to scale gradients. Ignoring momentum and bias correction for notation, a simplified Adam-style update and its associated diagonal preconditioner are
\begin{equation}
    \theta_{t+1} = \theta_t - \eta \frac{\hat g_t}{\sqrt{\hat v_t} + \epsilon},
\end{equation}
\begin{equation}
    \hat P_t = \mathrm{diag}(\sqrt{\hat v_t}+\epsilon).
\end{equation}
Here $\hat v_t$ estimates a second moment of gradients, and AdamW is the main representative of this family in our experiments.

Sophia-style optimizers also use diagonal preconditioning, but the preconditioner is based on a stochastic curvature or Hessian proxy rather than only the second moment of the gradient. In simplified notation, a Sophia-style update can be written as
\begin{equation}
    \theta_{t+1}
    =
    \theta_t
    -
    \eta \, \mathrm{clip}\left(\frac{\hat m_t}{\hat h_t+\lambda}, c_{\mathrm{clip}}\right),
\end{equation}
where $\hat h_t$ is a diagonal curvature estimate, $\lambda$ is damping, and $c_{\mathrm{clip}}$ is the ratio-clipping threshold. This gives a second setting where the inverse of a noisy diagonal quantity directly affects the update.

Matrix-preconditioned methods instead use a structured matrix estimate, such as a blockwise covariance, Fisher approximation, or inverse-root preconditioner. Shampoo is the representative matrix-preconditioned optimizer in our experiments. Unlike AdamW and Sophia, Shampoo uses matrix-valued preconditioners, making it a direct test of whether the proposed correction extends beyond coordinate-wise scaling.

\subsection{Problem setup}
\label{sec:problem-setup}

Fix a parameter value $\theta$ and let $z \sim \D$ be a training example. Let $g(z;\theta)$ denote the per-example gradient and let $P(z;\theta)$ denote the corresponding per-example or microbatch preconditioning statistic. We write the population gradient, population preconditioner, and target preconditioned update as
\begin{equation}
    \begin{aligned}
    g(z;\theta) &= \nabla_\theta \ell(\theta;z),
    &
    g(\theta) &= \E[g(z;\theta)] = \nabla f(\theta), \\
    P(\theta) &= \E[P(z;\theta)],
    &
    u^\star(\theta) &= P(\theta)^{-1} g(\theta).
    \end{aligned}
\end{equation}

Given a minibatch $A$, the corresponding empirical quantities and same-batch preconditioned update are
\begin{equation}
    \begin{aligned}
    \hat g_A &= \frac{1}{|A|}\sum_{z_i \in A} g(z_i;\theta),
    &
    \hat P_A &= \frac{1}{|A|}\sum_{z_i \in A} P(z_i;\theta), \\
    \hat u_A &= \hat P_A^{-1}\hat g_A.
    \end{aligned}
\end{equation}

Suppressing the shared $\theta$ argument, the central observation is that, in general,
\begin{equation}
    \E[\hat P_A^{-1}\hat g_A]
    \neq
    P^{-1}g.
    \label{eq:not-unbiased}
\end{equation}
This mismatch has two sources, described next.

\section{Two finite-sample biases}
\label{sec:two-biases}

Equation~\eqref{eq:not-unbiased} states that the same-batch preconditioned update is not an unbiased estimator of the target update. The right-hand side has two distinct sources of mismatch, made precise below: gradient--preconditioner coupling, and the nonlinearity of inversion.

\subsection{Coupling bias}
\label{sec:coupling-bias}

The first issue appears because the same random minibatch is used to estimate both the gradient and the preconditioner. Since both $\hat P_A$ and $\hat g_A$ are functions of the same random examples, they are statistically dependent:
\begin{equation}
    \E[\hat P_A^{-1}\hat g_A]
    \neq
    \E[\hat P_A^{-1}]\,\E[\hat g_A]
\end{equation}
in general. We call the resulting cross-moment term
\begin{equation}
    \mathrm{Bias}_{\mathrm{coupling}}
    =
    \E[\hat P_A^{-1}\hat g_A]
    -
    \E[\hat P_A^{-1}]\,\E[\hat g_A]
    \label{eq:coupling-bias}
\end{equation}
the gradient--preconditioner coupling bias.

This term is not present in plain SGD because $P=I$ is deterministic. It appears when the update uses a stochastic preconditioner, stochastic normalization term, stochastic clipping rule, or stochastic line-search estimate. It can also appear in AdamW, Sophia, and Shampoo because the same minibatch can influence both the descent direction and the stochastic geometry used to scale or precondition that direction.

The coupling term is not necessarily always harmful. In some settings, adapting the preconditioner to the current minibatch geometry may help. However, if the target is the population preconditioned update $P^{-1}g$, then same-batch coupling introduces an additional term that is not part of the target vector field.

\subsection{Inverse-preconditioner bias}
\label{sec:inverse-bias}

The second issue remains even if the gradient and preconditioner are estimated from independent samples. An unbiased preconditioner estimate, $\E[\hat P] = P$, does not imply unbiasedness after inversion because matrix inversion is nonlinear:
\begin{equation}
    \E[\hat P^{-1}]
    \neq
    \left(\E[\hat P]\right)^{-1}.
    \label{eq:inverse-bias}
\end{equation}
The scalar case makes the issue concrete: for the damped inverse $f(x)=(x+\lambda)^{-1}$ with $f''(x)=2(x+\lambda)^{-3}>0$, Jensen's inequality gives $\E[1/(\hat p+\lambda)]\geq 1/(p+\lambda)$, with strict inequality whenever $\hat p$ has nonzero variance. A simple two-point example, $\hat p\in\{0.5,1.5\}$ with $\E[\hat p]=1$, yields $\E[1/\hat p]=4/3\neq 1$. Appendix~\ref{app:delta-method-inverse} gives the full delta-method derivation of the leading bias and the corresponding correction.

For a preconditioned update, this creates the inverse-preconditioner bias
\begin{equation}
    \mathrm{Bias}_{\mathrm{inv}}
    =
    \left(\E[\hat P^{-1}] - P^{-1}\right)g.
    \label{eq:inverse-preconditioner-bias}
\end{equation}
This bias is reduced by larger batches, exponential moving averages, damping, and conservative learning rates. It is not necessarily the dominant error in all regimes. However, it is a real finite-sample effect and becomes more relevant when preconditioner estimates are noisy, aggressive, or matrix-valued. This is why the issue is especially natural to test on AdamW, Sophia, and Shampoo.

\section{Proposed method}
\label{sec:method}

We propose a two-part correction framework. The first component removes gradient--preconditioner coupling by estimating the gradient and preconditioner from independent microbatch groups. The second component reduces inverse-preconditioner bias using single-batch microbatch variability.

\subsection{Cross-fitted preconditioning}

At step $t$, parameters are fixed at $\theta_t$. We sample two independent microbatch groups $A$ and $B$ from the same training distribution and use them for the numerator and preconditioner estimates,
\begin{equation}
    \hat g_A
    =
    \frac{1}{|A|}\sum_{z_i\in A} g(z_i;\theta_t),
    \qquad
    \hat P_B
    =
    \frac{1}{|B|}\sum_{z_i\in B} P(z_i;\theta_t).
\end{equation}
The resulting cross-fitted update and, by independence, its factorized expectation are
\begin{equation}
    \hat u_{\mathrm{cf}}
    =
    \hat P_B^{-1}\hat g_A.
    \label{eq:crossfit-update}
\end{equation}
\begin{equation}
    \E[\hat P_B^{-1}\hat g_A]
    =
    \E[\hat P_B^{-1}]\,\E[\hat g_A].
\end{equation}
Therefore, cross-fitting removes the coupling term in Equation~\eqref{eq:coupling-bias}. It does not, by itself, remove inverse-preconditioner bias. The remaining mismatch is
\begin{equation}
    \E[\hat u_{\mathrm{cf}}] - P^{-1}g
    =
    \left(\E[\hat P_B^{-1}] - P^{-1}\right)g.
\end{equation}

In practice, $A$ and $B$ can be created by splitting the current training batch, by using separate gradient-accumulation microbatch groups, or by drawing two independent minibatches from the data stream. The theory assumes independence. Occasional overlap is acceptable if the sampling procedures are independent and the dataset is large.

\subsection{Single-batch inverse correction}

We next reduce the finite-sample bias in $\hat P_B^{-1}$. The correction must be computed at the current parameter value $\theta_t$, because training does not revisit the same point. We therefore use microbatch variability inside the current batch.

Let $B$ be split into $m$ microbatches $B_1,\ldots,B_m$, with microbatch preconditioners and averaged preconditioner
\begin{equation}
    \hat P_j = \hat P_{B_j}.
\end{equation}
\begin{equation}
    \bar P = \frac{1}{m}\sum_{j=1}^m \hat P_j.
\end{equation}
A standard method would use $(\bar P+\lambda I)^{-1}$, where $\lambda>0$ is a damping term. We instead use the variation among $\hat P_1,\ldots,\hat P_m$ to estimate and subtract the leading inverse bias.

\subsubsection{Diagonal correction for AdamW and Sophia}

For diagonal preconditioners, the correction is elementwise. This covers AdamW and Sophia. Let $\bar p_k$ denote the $k$th coordinate of the averaged preconditioner and let $\widehat{\Var}(\bar p_k)$ denote the estimated variance of the microbatch mean for that coordinate. The relevant inverse functional is
\begin{equation}
    T(p) = \frac{1}{p+\lambda}.
\end{equation}
Its second-order expansion gives
\begin{equation}
    \E[T(\bar p_k)]
    \approx
    T(p_k)
    +
    \frac{1}{2}T''(p_k)\Var(\bar p_k).
\end{equation}
For this functional, the curvature and leading finite-sample bias are
\begin{equation}
    T''(p) = \frac{2}{(p+\lambda)^3},
    \qquad
    \frac{\Var(\bar p_k)}{(p_k+\lambda)^3}.
\end{equation}
Replacing unknown population quantities with batch estimates gives the corrected inverse
\begin{equation}
    \widetilde T(\bar p_k)
    =
    \frac{1}{\bar p_k+\lambda}
    -
    \frac{\widehat{\Var}(\bar p_k)}{(\bar p_k+\lambda)^3}.
    \label{eq:diag-correction}
\end{equation}
The variance term is the variance of the microbatch mean, not the variance of a single microbatch statistic. If $p_{j,k}$ is the $k$th preconditioner entry from microbatch $B_j$, then
\begin{equation}
    \widehat{\Var}(\bar p_k)
    =
    \frac{1}{m(m-1)}
    \sum_{j=1}^m (p_{j,k}-\bar p_k)^2,
    \qquad
    \bar p_k=\frac{1}{m}\sum_{j=1}^m p_{j,k}.
    \label{eq:variance-of-mean-estimator}
\end{equation}
The factor $m-1$ gives the unbiased sample variance across microbatches, and the additional factor $m$ converts the variance of microbatch statistics into the variance of their mean. For numerical stability, we clip the corrected inverse before applying the update:
\begin{equation}
    \widetilde T(\bar p_k)
    \leftarrow
    \max\left\{0, \widetilde T(\bar p_k)\right\}.
\end{equation}
The resulting bias-corrected update is
\begin{equation}
    \hat u_{\mathrm{bc}}
    =
    \widetilde P_B^{-1}\hat g_A,
    \label{eq:bias-corrected-update}
\end{equation}
where $\widetilde P_B^{-1}$ is the diagonal matrix with entries given by Equation~\eqref{eq:diag-correction}. Appendix~\ref{app:delta-method-inverse} gives the delta-method derivation and explains why replacing the unknown population value $p_k$ by the observed batch estimate $\bar p_k$ introduces only higher-order error.

This correction is cheap. It requires storing the microbatch mean and variance of the preconditioner estimate, not recomputing multiple inverses. For AdamW, $p$ corresponds to the denominator induced by the second-moment estimate, such as $p_k=\sqrt{v_k}$. For Sophia, $p$ corresponds to the diagonal curvature or Hessian-proxy estimate used in the denominator.

\subsubsection{Jackknife correction}
\label{app:jackknife-correction}

A more general single-batch correction uses the jackknife. The inverse functional, leave-one-microbatch-out preconditioner, and jackknife-corrected inverse are
\begin{equation}
    T(P) = (P+\lambda I)^{-1}.
\end{equation}
\begin{equation}
    \bar P_{(-j)}
    =
    \frac{1}{m-1}\sum_{\ell \neq j}\hat P_\ell
\end{equation}
\begin{equation}
    T_{\mathrm{jack}}
    =
    m T(\bar P)
    -
    \frac{m-1}{m}\sum_{j=1}^m T(\bar P_{(-j)}).
    \label{eq:jackknife}
\end{equation}
Under standard smoothness conditions for the functional $T$, the jackknife cancels the leading $O(1/m)$ finite-sample bias term, leaving a smaller higher-order bias. This correction is more general than the diagonal Taylor correction, but it can be more expensive because it requires evaluating the inverse functional on $m$ leave-one-out estimates.

For diagonal preconditioners, Equation~\eqref{eq:jackknife} is inexpensive and can be used directly. For matrix preconditioners such as Shampoo, it may be too expensive to compute every step. In that case, the jackknife can be applied periodically, blockwise, or only to small structured preconditioner blocks.

\subsubsection{Matrix-preconditioner correction for Shampoo}

For Shampoo, the preconditioner is matrix-valued or block-matrix-valued, and the correction can be applied in an eigenbasis. Suppose
\begin{equation}
    \bar P = Q\Lambda Q^\top.
\end{equation}
We project each microbatch preconditioner into this basis,
\begin{equation}
    \hat P_j' = Q^\top \hat P_j Q.
\end{equation}
The diagonal entries of $\hat P_j'$ estimate variability along the eigen-directions of $\bar P$. Let $\lambda_k$ be the $k$th eigenvalue of $\bar P$ and let $\widehat{\Var}(\bar \lambda_k)$ be the estimated variance of the corresponding projected microbatch mean. For an inverse-root preconditioner, define the scalar spectral map
\begin{equation}
    T_\alpha(x)=(x+\lambda_0)^{-\alpha}.
\end{equation}
For matrix parameters in Shampoo, $\alpha=1/4$. The corrected inverse-root eigenvalue and preconditioner are
\begin{equation}
    \widetilde d_k
    =
    T_\alpha(\lambda_k)
    -
    \frac{1}{2}T_\alpha''(\lambda_k)\widehat{\Var}(\bar \lambda_k).
\end{equation}
\begin{equation}
    \widetilde P^{-\alpha}
    =
    Q\,\mathrm{diag}(\widetilde d_1,\ldots,\widetilde d_d)\,Q^\top.
\end{equation}
This avoids repeated eigendecompositions. It uses the eigendecomposition already required by Shampoo-style matrix preconditioning and adds only a variance estimate in the current eigenbasis. Appendix~\ref{app:inverse-root-proof} gives the bias and residual-error analysis for the inverse-root correction, including the eigenbasis-of-$\bar P$ approximation.

\paragraph{Interpretation: numerator vs.\ denominator.}
The two correction components can be interpreted by comparing numerator and denominator signals. When both the numerator and denominator are large, or both small, the cross-fitted update resembles same-batch behavior: gradient magnitude and preconditioner agree about the local scale of the step. The informative cases are asymmetric: a large numerator paired with a small independent denominator is unsupported, and the inverse-bias correction shrinks the inverse denominator in high-variance directions, keeping the step conservative; a small numerator paired with a large denominator is benign because little gradient signal is available. The correction therefore damps unsupported large steps while leaving reliable large-gradient, large-denominator cases controlled.

Algorithm~\ref{alg:bc-step} summarizes the proposed method. The convergence theorems are given in Section~\ref{sec:convergence}, and detailed optimizer instantiations are given in Appendix~\ref{sec:algorithm}.

\begin{algorithm}[t]
\caption{Bias-corrected preconditioned step at parameters $\theta_t$.}
\label{alg:bc-step}
\begin{algorithmic}[1]
\Require parameters $\theta_t$; batch $\mathcal{B}_t$; microbatch count $m$; damping $\lambda$; learning rate $\eta$
\State Split $\mathcal{B}_t$ into two independent groups $A$ (numerator) and $B$ (denominator)
\State $\hat g_A \gets \frac{1}{|A|}\sum_{z\in A}\nabla\ell(\theta_t;z)$ \Comment{numerator gradient estimate}
\State Split $B$ into microbatches $B_1,\ldots,B_m$
\For{$j=1,\ldots,m$}
    \State Compute microbatch preconditioner $\hat P_j \gets \hat P_{B_j}(\theta_t)$
\EndFor
\State $\bar P_B \gets \frac{1}{m}\sum_{j=1}^m \hat P_j$ \Comment{averaged denominator}
\If{diagonal preconditioner (AdamW, Sophia)}
    \State $\widehat{\Var}(\bar P_B) \gets \frac{1}{m(m-1)}\sum_{j=1}^m (\hat P_j - \bar P_B)^{\odot 2}$
    \State $\widetilde H_B \gets (\bar P_B+\lambda)^{-1} - \widehat{\Var}(\bar P_B)\,(\bar P_B+\lambda)^{-3}$ \Comment{Eq.~\eqref{eq:diag-correction}}
\Else \Comment{matrix preconditioner (Shampoo)}
    \State Eigendecompose $\bar P_B = Q\Lambda Q^\top$
    \State Project $\hat P_j$ into the eigenbasis and estimate $\widehat{\Var}(\bar\lambda_k)$ for each eigen-direction
    \State Apply Eq.~\eqref{eq:invroot-correction} eigen-direction-wise to obtain the corrected inverse-root operator $\widetilde H_B$
\EndIf
\State $\theta_{t+1} \gets \theta_t - \eta\, \widetilde H_B\,\hat g_A$
\end{algorithmic}
\end{algorithm}

\section{Convergence analysis}
\label{sec:convergence}

For a nonconvex language-model loss, one generally cannot prove that an optimizer reaches a better global optimum. A correct statement is more local and estimator-based: the bias-corrected update is designed to reduce the targeted stochastic-update bias components relative to the standard same-batch preconditioned update; under smooth nonconvex assumptions, smaller update bias improves the stationarity bound; and under a PL or locally strongly-convex-type condition, smaller update bias improves the linear contraction constant and the limiting suboptimality floor. This is consistent with biased-SGD theory, where estimator bias changes attainable accuracy and convergence constants \citep{ajalloeian2020convergence}, and with existing adaptive and preconditioned optimizer analyses that prove convergence under smoothness, bounded variance, and bounded preconditioner assumptions.

We analyze a generic stochastic preconditioned update
\begin{equation}
    \theta_{t+1}
    =
    \theta_t-\eta \hat u_t,
    \qquad
    \hat u_t
    =
    \widehat H_t \hat g_t,
\end{equation}
where the target population update is
\begin{equation}
    u_t^\star
    =
    H_t \nabla f(\theta_t),
    \qquad
    H_t=P_t^{-1}.
\end{equation}
For AdamW, $H_t$ is a diagonal inverse second-moment preconditioner; for Sophia, it is a diagonal inverse curvature preconditioner; and for Shampoo, it is a structured inverse-root matrix preconditioner.

\paragraph{Conditional update bias.}
Let $\mathcal F_t$ denote the history before sampling minibatches at step $t$. Define the conditional stochastic-update bias
\begin{equation}
    b_t
    =
    \E[\hat u_t\mid \mathcal F_t]
    -
    H_t\nabla f(\theta_t),
\end{equation}
and the conditional update variance
\begin{equation}
    \sigma_t^2
    =
    \E\!\left[
    \left\|
    \hat u_t-\E[\hat u_t\mid\mathcal F_t]
    \right\|^2
    \mid \mathcal F_t
    \right].
\end{equation}

Theorem~\ref{thm:bias-reduction} below makes the connection to Section~\ref{sec:method} precise: cross-fitting removes the coupling component of $b_t$, and delta-method correction reduces the inverse component of $b_t$ from $O(m^{-1})$ to $O(m^{-3/2})$. Theorems~\ref{thm:nonconvex} and~\ref{thm:pl} then translate this estimator-level bias reduction into convergence guarantees under standard assumptions.

\begin{theorem}[Bias reduction by cross-fitting and inverse correction]
\label{thm:bias-reduction}
Fix $\theta_t$. Let $A$ and $B$ be independent minibatch groups sampled from the same data distribution. Let
\begin{equation}
    \hat g_A
    \quad\text{and}\quad
    \widehat H_B
\end{equation}
be the stochastic gradient and inverse-preconditioner estimates, respectively. Suppose
\begin{equation}
    \E[\hat g_A\mid\mathcal F_t]=\nabla f(\theta_t),
\end{equation}
and suppose the inverse-preconditioner map is three-times continuously differentiable in a neighborhood of the population preconditioner, with bounded third derivative and bounded microbatch moments. If the preconditioner is estimated from $m$ independent microbatches, then the standard same-batch estimator satisfies
\begin{align}
    &\E[\widehat H_A\hat g_A\mid\mathcal F_t]
    -
    H_t\nabla f(\theta_t) \nonumber \\
    &\qquad =
    \underbrace{
    \mathrm{Cov}(\widehat H_A,\hat g_A\mid\mathcal F_t)
    }_{\text{coupling bias}}
    +
    \underbrace{
    \left(\E[\widehat H_A\mid\mathcal F_t]-H_t\right)
    \nabla f(\theta_t)
    }_{\text{inverse bias}},
\end{align}
where $\mathrm{Cov}(\widehat H_A,\hat g_A\mid\mathcal F_t)$ denotes the cross-moment
$\E[\widehat H_A\hat g_A\mid\mathcal F_t]
-\E[\widehat H_A\mid\mathcal F_t]\E[\hat g_A\mid\mathcal F_t]$.
The cross-fitted estimator removes the coupling term:
\begin{equation}
    \E[\widehat H_B\hat g_A\mid\mathcal F_t]
    -
    H_t\nabla f(\theta_t)
    =
    \left(\E[\widehat H_B\mid\mathcal F_t]-H_t\right)
    \nabla f(\theta_t).
\end{equation}
If the delta-method correction is applied to the inverse-preconditioner estimate, then, for diagonal preconditioners or for matrix preconditioners whose microbatch perturbations are diagonal in a fixed population-aligned eigendirection basis,
\begin{equation}
    \E[\widetilde H_B\mid\mathcal F_t]
    =
    H_t + O(m^{-3/2}),
\end{equation}
whereas the uncorrected inverse estimator has
\begin{equation}
    \E[\widehat H_B\mid\mathcal F_t]
    =
    H_t + O(m^{-1}).
\end{equation}
Consequently,
\begin{equation}
    \left\|
    \E[\widetilde H_B\hat g_A\mid\mathcal F_t]
    -
    H_t\nabla f(\theta_t)
    \right\|
    =
    O(m^{-3/2})\|\nabla f(\theta_t)\|.
\end{equation}
\end{theorem}

For Shampoo, the implementation applies the same scalar correction in the empirical eigenbasis of the averaged preconditioner; Appendix~\ref{app:inverse-root-proof} states the resulting basis-rotation approximation explicitly.

\begin{proof}
For the same-batch estimator,
\begin{equation}
    \E[\widehat H_A\hat g_A\mid\mathcal F_t]
    =
    \E[\widehat H_A\mid\mathcal F_t]\E[\hat g_A\mid\mathcal F_t]
    +
    \mathrm{Cov}(\widehat H_A,\hat g_A\mid\mathcal F_t).
\end{equation}
Since $\E[\hat g_A\mid\mathcal F_t]=\nabla f(\theta_t)$, subtracting
$H_t\nabla f(\theta_t)$ gives the stated decomposition.

For cross-fitting, $\widehat H_B$ and $\hat g_A$ are conditionally independent given $\mathcal F_t$, so
\begin{equation}
    \E[\widehat H_B\hat g_A\mid\mathcal F_t]
    =
    \E[\widehat H_B\mid\mathcal F_t]
    \E[\hat g_A\mid\mathcal F_t].
\end{equation}
Thus the covariance term vanishes.

For the inverse correction, write the scalar inverse or inverse-root map as $T(x)$. If $\bar x$ is the mean of $m$ microbatch estimates with mean $x$, then Taylor expansion gives
\begin{equation}
    \E[T(\bar x)]
    =
    T(x)
    +
    \frac{1}{2}T''(x)\Var(\bar x)
    +
    O(m^{-3/2}).
\end{equation}
The delta-method estimator subtracts
\begin{equation}
    \frac{1}{2}T''(\bar x)\widehat{\Var}(\bar x),
\end{equation}
whose expectation equals the leading bias term up to $O(m^{-3/2})$. Therefore the corrected inverse estimate has residual bias $O(m^{-3/2})$. Applying this coordinatewise, or eigendirection-wise under the commuting perturbation condition in the theorem statement, gives the result.
\end{proof}

The estimator-level bias reduction in Theorem~\ref{thm:bias-reduction} is the input to the convergence analysis below. We now state the assumptions used by the convergence theorems and follow with the theorems themselves. Long proofs are deferred to Appendix~\ref{app:convergence-proofs}; we keep one-paragraph sketches in the body.

\begin{assumption}[Smoothness and bounded preconditioning]
\label{ass:smooth}
The objective $f$ is $L$-smooth and lower bounded by $f_\star$. The population inverse preconditioners $H_t$ are symmetric positive definite and satisfy
\begin{equation}
    \mu_H I \preceq H_t \preceq M_H I
\end{equation}
for constants $0<\mu_H\le M_H<\infty$. The stochastic update has bounded relative bias and bounded conditional variance:
\begin{equation}
    \begin{aligned}
    \left\|
    \E[\hat u_t\mid\mathcal F_t]
    -
    H_t\nabla f(\theta_t)
    \right\|
    \le
    \rho_{\mathrm{bias}}\,\mu_H\|\nabla f(\theta_t)\|,
    \qquad
    0\le \rho_{\mathrm{bias}}<1,\\
    \E\!\left[
    \left\|
    \hat u_t-\E[\hat u_t\mid\mathcal F_t]
    \right\|^2
    \mid\mathcal F_t
    \right]
    \le
    \sigma^2.
    \end{aligned}
\end{equation}
\end{assumption}

\begin{assumption}[PL condition]
\label{ass:pl}
The objective satisfies the Polyak--Lojasiewicz condition with constant $\mu_{\mathrm{PL}}>0$:
\begin{equation}
    \frac{1}{2}\|\nabla f(\theta)\|^2
    \ge
    \mu_{\mathrm{PL}}(f(\theta)-f_\star)
    \quad
    \text{for all iterates } \theta.
\end{equation}
\end{assumption}

The parameter $\rho_{\mathrm{bias}}$ in Assumption~\ref{ass:smooth} is the relative update-bias level. Theorem~\ref{thm:bias-reduction} shows how cross-fitting and inverse correction reduce the bias terms that contribute to $\rho_{\mathrm{bias}}$, up to the variance introduced by sample splitting.

\begin{theorem}[Nonconvex stationarity bound]
\label{thm:nonconvex}
Under Assumption~\ref{ass:smooth}, choose a constant step size $\eta$ satisfying
\begin{equation}
    0<\eta
    \le
    \frac{\mu_H(1-\rho_{\mathrm{bias}})}
    {L(M_H+\rho_{\mathrm{bias}}\mu_H)^2}.
\end{equation}
Then after $T$ steps,
\begin{equation}
    \frac{1}{T}
    \sum_{t=0}^{T-1}
    \E\|\nabla f(\theta_t)\|^2
    \le
    \frac{2(f(\theta_0)-f_\star)}
    {\eta T\mu_H(1-\rho_{\mathrm{bias}})}
    +
    \frac{L\eta\sigma^2}
    {\mu_H(1-\rho_{\mathrm{bias}})}.
\end{equation}
\end{theorem}

\paragraph{Proof sketch.}
By $L$-smoothness, $f(\theta_{t+1})\le f(\theta_t)-\eta\langle\nabla f(\theta_t),\hat u_t\rangle+\frac{L\eta^2}{2}\|\hat u_t\|^2$. Taking conditional expectation, the bounded relative bias gives $\langle\nabla f,\E[\hat u_t\mid\mathcal F_t]\rangle\ge\mu_H(1-\rho_{\mathrm{bias}})\|\nabla f\|^2$, while the bounded conditional variance gives $\E[\|\hat u_t\|^2\mid\mathcal F_t]\le(M_H+\rho_{\mathrm{bias}}\mu_H)^2\|\nabla f\|^2+\sigma^2$. Combining yields a per-step descent inequality with coefficient $c_{\mathrm{bias}}=\eta\mu_H(1-\rho_{\mathrm{bias}})-\frac{L\eta^2}{2}(M_H+\rho_{\mathrm{bias}}\mu_H)^2$, and the step-size condition implies $c_{\mathrm{bias}}\ge\frac{1}{2}\eta\mu_H(1-\rho_{\mathrm{bias}})$. Summing over $t=0,\ldots,T-1$ and using $f(\theta_T)\ge f_\star$ gives the claim. The full argument is in Appendix~\ref{app:convergence-proofs}.

This theorem formalizes the estimator-level convergence statement. If the corrected method has $\rho_{\mathrm{bias},\mathrm{BC}}<\rho_{\mathrm{bias},\mathrm{std}}$ and does not increase $\sigma^2$ too much, then the upper bound is smaller. If sample splitting increases variance, the theorem also explains why an optimizer can fail under fixed compute: $\rho_{\mathrm{bias}}$ improves, but $\sigma^2$ worsens.

\begin{theorem}[Linear convergence to a smaller neighborhood]
\label{thm:pl}
Under Assumptions~\ref{ass:smooth} and~\ref{ass:pl}, with the step size of Theorem~\ref{thm:nonconvex}, define
\begin{equation}
    c_{\mathrm{bias}}
    =
    \eta\mu_H(1-\rho_{\mathrm{bias}})
    -
    \frac{L\eta^2}{2}(M_H+\rho_{\mathrm{bias}}\mu_H)^2
    >
    0.
\end{equation}
Assume also that the step size is small enough that $0<2\mu_{\mathrm{PL}}c_{\mathrm{bias}}<1$. Then
\begin{equation}
    \E[f(\theta_T)-f_\star]
    \le
    (1-2\mu_{\mathrm{PL}}c_{\mathrm{bias}})^T
    (f(\theta_0)-f_\star)
    +
    \frac{L\eta^2\sigma^2}
    {4\mu_{\mathrm{PL}}c_{\mathrm{bias}}}.
\end{equation}
\end{theorem}

\paragraph{Proof sketch.}
The per-step descent inequality from Theorem~\ref{thm:nonconvex} combined with the PL condition $\|\nabla f(\theta_t)\|^2\ge 2\mu_{\mathrm{PL}}(f(\theta_t)-f_\star)$ gives a linear contraction $\E[f(\theta_{t+1})-f_\star\mid\mathcal F_t]\le(1-2\mu_{\mathrm{PL}}c_{\mathrm{bias}})(f(\theta_t)-f_\star)+\frac{L\eta^2\sigma^2}{2}$. Unrolling the recursion produces the stated geometric decay plus an asymptotic neighborhood. The full argument is in Appendix~\ref{app:convergence-proofs}.

This theorem does not claim a better global optimum in a nonconvex language model. It says that the method converges to a smaller objective neighborhood when the corrected estimator has a better effective bias--variance tradeoff.

\begin{corollary}[When bias correction improves convergence]
\label{cor:bc-better}
Let standard and bias-corrected updates satisfy Assumption~\ref{ass:smooth} with a common step size and parameters
\begin{equation}
    (\rho_{\mathrm{bias},\mathrm{std}},\sigma_{\mathrm{std}}^2)
    \quad\text{and}\quad
    (\rho_{\mathrm{bias},\mathrm{BC}},\sigma_{\mathrm{BC}}^2).
\end{equation}
Let $c_{\mathrm{std}}$ and $c_{\mathrm{BC}}$ denote the corresponding constants from Theorem~\ref{thm:pl}. If
\begin{equation}
    c_{\mathrm{BC}}>c_{\mathrm{std}},
\end{equation}
and both methods satisfy the contraction condition in Theorem~\ref{thm:pl}, then the PL contraction factor is better for the bias-corrected method:
\begin{equation}
    1-2\mu_{\mathrm{PL}}c_{\mathrm{BC}}
    <
    1-2\mu_{\mathrm{PL}}c_{\mathrm{std}}.
\end{equation}
If additionally
\begin{equation}
    \frac{\sigma_{\mathrm{BC}}^2}{c_{\mathrm{BC}}}
    <
    \frac{\sigma_{\mathrm{std}}^2}{c_{\mathrm{std}}},
\end{equation}
then the limiting suboptimality neighborhood is also smaller for the bias-corrected method.
\end{corollary}

The direct theoretical conclusion is therefore not that the corrected optimizer always reaches a better global optimum. Rather, under the stated assumptions, bias correction reduces the targeted stochastic-update bias components; this improves the convergence constant and the limiting optimization neighborhood whenever the variance cost of correction is not larger than the bias-reduction benefit. This unified statement applies at the level of the preconditioned update and therefore covers AdamW, Sophia, and Shampoo through the corresponding choices of $H_t$ when the implemented update satisfies the bounded-bias and bounded-variance conditions.

\section{Experiments}
\label{sec:experiments}

We evaluate the correction in two regimes (Table~\ref{tab:experiment-setup}). First, we run language-model pretraining from random initialization using the Qwen2.5-0.5B architecture on packed FineWeb-Edu sequences --- a longer and noisier setting, and therefore the more direct stress test for finite-sample preconditioner bias. We also run a mixed-quality pretraining diagnostic in which 20\% of the training sequences are corrupted by span replacement while the held-out evaluation set remains clean. Second, we run instruction tuning from a pretrained Qwen2.5-0.5B checkpoint on Alpaca-style supervised data --- a short-horizon, relatively stable post-training setting. Unless otherwise stated, reported values are held-out cross-entropy losses and lower is better. The primary experimental axes are the optimizer family, the bias-correction variant, and the training regime. We instantiate the method for AdamW, Sophia-G, and Shampoo. For each family, the main variants are standard same-batch preconditioning, cross-fit only, inverse correction only, and full BC. The main component ablation is reported in Appendix~\ref{app:bias-correction-ablations}. Additional SFT studies sweep learning rate, batch size, cross-fit mixing, pre-EMA versus post-EMA inverse correction, warm-starting from standard training, Shampoo layer coverage, clipping, preconditioner recomputation frequency, and the number of denominator microbatches; these sweep results are reported in Appendix~\ref{app:additional-experiments}.

\begin{table}[t]
\centering
\footnotesize
\setlength{\tabcolsep}{2pt}
\caption{Experimental regimes used in this study.}
\label{tab:experiment-setup}
\begin{tabular}{p{0.17\linewidth}p{0.22\linewidth}p{0.27\linewidth}p{0.23\linewidth}}
\toprule
\rowcolor{headergray}
\textbf{Regime} & \textbf{Initialization} & \textbf{Training Data} & \textbf{Evaluation} \\
\midrule
Clean pretraining & Qwen2.5-0.5B config, random init & 256K packed FineWeb-Edu sequences & 10K packed sequences, 10.24M tokens \\
Mixed-quality pretraining & Qwen2.5-0.5B config, random init & 80\% clean + 20\% span-replaced train sequences & Same clean 10K packed evaluation set \\
Instruction tuning & Qwen2.5-0.5B pretrained & 32K supervised examples & 500 examples; selected checkpoints also 5K \\
\bottomrule
\end{tabular}
\end{table}

For pretraining, FineWeb-Edu text is tokenized with the Qwen tokenizer and packed into fixed 1024-token sequences; the pretraining experiments use 256K packed training sequences and evaluate on 10K packed sequences. In the mixed-quality diagnostic, 51,200 of the 256,000 training sequences are selected for corruption. For each selected sequence, approximately 30\% of its 64-token blocks are replaced by spans sampled from other training sequences, preserving token-level plausibility while breaking semantic coherence; the evaluation set is copied unchanged. For instruction tuning, we shuffle the Alpaca-style supervised dataset with a fixed seed and reserve held-out examples before selecting the 32K-example training subset. The main SFT evaluation uses 500 held-out examples, and selected saved checkpoints are also evaluated on a larger 5K-example held-out set. Table~\ref{tab:training-config} gives the optimizer hyperparameters, microbatch structure, learning-rate schedule, random seeds, and compute resources. Learning rates and implementation variants were chosen through the targeted sweeps reported in Appendix~\ref{app:additional-experiments}; the reported main comparisons are the settings most relevant to the claims rather than an exhaustive hyperparameter grid.

\subsection{Main pretraining results}

Table~\ref{tab:pretraining-results} reports pretraining results with final held-out evaluation. Bias correction improves the corresponding standard optimizer for all three optimizer families in these pretraining comparisons. For AdamW, the main corrected run uses the LOO+Jensen variant described in Section~\ref{sec:algorithm}: embeddings follow standard AdamW, while dense parameters use leave-one-out cross-fitting over 64 microbatches and the Jensen inverse correction. On clean FineWeb-Edu pretraining, this reduces eval loss from 4.836 to 4.687, a reduction of 0.149 nats. In the mixed-quality setting, AdamW LOO+Jensen reduces clean eval loss from 4.8225 to 4.8034. This second gap is smaller (0.0191 nats) and should be read as a directional diagnostic rather than a large effect; notably, the corrected run has higher final training loss, consistent with fitting the corrupted spans less aggressively while retaining a lower clean held-out loss. Sophia full BC reduces eval loss from 6.6647 to 6.5946, and Shampoo full BC reduces eval loss from 5.7916 to 5.6813.

\begin{table}[t]
\centering
\footnotesize
\setlength{\tabcolsep}{2pt}
\caption{Pretraining results from random initialization on packed FineWeb-Edu sequences. AdamW LOO+Jensen uses 64 leave-one-out folds of 8 examples each; embeddings use standard AdamW and dense parameters use LOO bias correction. The noisy-train rows use the span-replacement training set and the same clean evaluation set.}
\label{tab:pretraining-results}
\resizebox{\textwidth}{!}{%
\begin{tabular}{p{0.10\linewidth}p{0.15\linewidth}p{0.18\linewidth}p{0.12\linewidth}p{0.13\linewidth}ccc}
\toprule
\rowcolor{headergray}
\textbf{Optimizer} & \textbf{Train Data} & \textbf{Variant} & \textbf{Batch / Split} & \textbf{LR} & \textbf{Final-50 Train} & \textbf{Eval Loss} & \textbf{Delta} \\
\midrule
\multicolumn{8}{l}{\textit{AdamW}} \\
AdamW & Clean & Standard & b512 & $6\mathrm{e}{-4}$ & 4.832 & 4.836 &  \\
AdamW & Clean & LOO+Jensen BC & b512, LOO-64 & embed $6\mathrm{e}{-4}$; dense $9\mathrm{e}{-4}$ & \textbf{4.733} & \textbf{4.687} & \textbf{-0.149} \\
AdamW & Noisy & Standard & b512 & $6\mathrm{e}{-4}$ & \textbf{4.853} & 4.823 &  \\
AdamW & Noisy & LOO+Jensen BC & b512, LOO-64 & embed $6\mathrm{e}{-4}$; dense $9\mathrm{e}{-4}$ & 4.872 & \textbf{4.803} & \textbf{-0.019} \\
\addlinespace[2pt]
\multicolumn{8}{l}{\textit{Sophia-G}} \\
Sophia & Clean & Standard & b512 & $6\mathrm{e}{-4}$ & 6.663 & 6.665 &  \\
Sophia & Clean & Full BC & A512/B512 & $6\mathrm{e}{-4}$ & \textbf{6.592} & \textbf{6.595} & \textbf{-0.070} \\
\addlinespace[2pt]
\multicolumn{8}{l}{\textit{Shampoo}} \\
Shampoo & Clean & Standard & b512 & $6\mathrm{e}{-4}$ & 5.787 & 5.792 &  \\
Shampoo & Clean & Full BC & A512/B512 & $1\mathrm{e}{-3}$ & \textbf{5.676} & \textbf{5.681} & \textbf{-0.110} \\
\bottomrule
\end{tabular}
}
\end{table}

The AdamW result illustrates why the batch split matters. Two-fold AdamW full-BC uses only half the batch for the denominator estimate, which inflates the square-of-mean denominator noise and makes updates too conservative. The LOO version keeps the denominator scale close to standard AdamW because each fold's denominator uses 504 of the 512 examples. It therefore removes same-microbatch coupling without paying the same small-denominator penalty.

\begin{figure}[!t]
\centering
\includegraphics[width=0.88\linewidth]{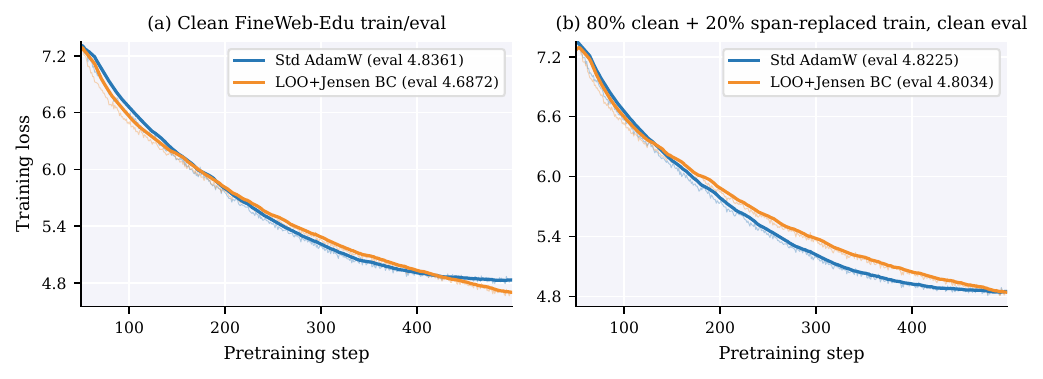}
\caption{AdamW pretraining comparisons for Qwen2.5-0.5B trained from random initialization. Left: clean FineWeb-Edu training and clean held-out evaluation. Right: mixed-quality training with 20\% span-replaced sequences and the same clean held-out evaluation. Thin lines show per-step losses and thick lines show a trailing moving average. LOO+Jensen BC improves clean held-out eval in both settings; in the mixed-quality setting it does so despite higher final training loss, consistent with less fitting of corrupted spans.}
\label{fig:adamw-loo-pretrain}
\end{figure}

\begin{table}[!t]
\centering
\footnotesize
\setlength{\tabcolsep}{2pt}
\caption{Best instruction-tuning comparisons. Values are held-out cross-entropy losses. Deltas that round to $0.000$ indicate differences below $5\times 10^{-4}$.}
\label{tab:sft-main}
\begin{tabular}{p{0.11\linewidth}p{0.20\linewidth}p{0.20\linewidth}cccc}
\toprule
\rowcolor{headergray}
\textbf{Optimizer} & \textbf{Standard Setting} & \textbf{Corrected Setting} & \textbf{Eval Set} & \textbf{Std} & \textbf{Full BC} & \textbf{Delta} \\
\midrule
AdamW & b512, lr $2\mathrm{e}{-5}$ & b512, lr $1\mathrm{e}{-4}$ & 500 & 1.347 & \textbf{1.346} & \textbf{-0.001} \\
\addlinespace[1pt]
Sophia & b512, lr $2\mathrm{e}{-5}$ & b512, lr $2\mathrm{e}{-5}$ & 5K & 1.342 & \textbf{1.342} & 0.000 \\
\addlinespace[1pt]
Shampoo & attn only, b512, lr $2\mathrm{e}{-5}$ & attn only, b512, lr $2\mathrm{e}{-5}$ & 500 & 1.347 & \textbf{1.347} & 0.000 \\
\bottomrule
\end{tabular}
\end{table}

\begin{figure}[!t]
\centering
\includegraphics[width=0.72\linewidth]{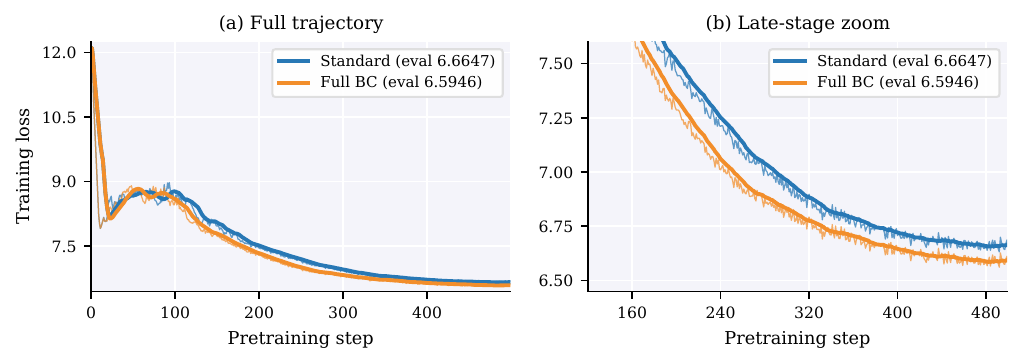}
\caption{Sophia pretraining loss curves for Qwen2.5-0.5B trained from random initialization. Thin lines show per-step losses and thick lines show a trailing moving average. Full BC reaches both lower final-50 training loss and lower held-out eval loss than the standard optimizer in this pretraining comparison.}
\label{fig:sophia-pretrain-curve}
\end{figure}

\begin{figure}[!t]
\centering
\includegraphics[width=0.72\linewidth]{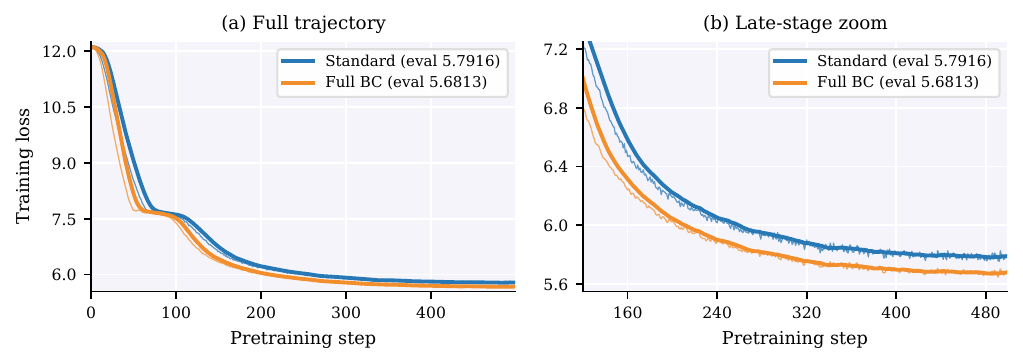}
\caption{Shampoo pretraining loss curves for Qwen2.5-0.5B trained from random initialization with attention and MLP layers using Shampoo preconditioning. Thin lines show per-step losses and thick lines show a trailing moving average. The full-BC run reaches lower final-50 training loss and lower held-out eval loss than the standard Shampoo run.}
\label{fig:shampoo-pretrain-curve}
\end{figure}

\subsection{Main instruction-tuning results}

Table~\ref{tab:sft-main} summarizes the strongest instruction-tuning comparisons. The differences are small, as expected for a short SFT run starting from a strong pretrained model. These results should be read as targeted evidence rather than an exhaustive grid: AdamW benefits most after retuning and pre-EMA inverse correction, Sophia is nearly tied on 500 examples but slightly better on the larger 5K evaluation, and Shampoo is effectively neutral in the attention-only SFT setting.

\section{Limitations}
\label{sec:limitations}

This study evaluates one model scale, Qwen2.5-0.5B, in two training regimes: FineWeb-Edu pretraining from random initialization and supervised instruction tuning on an Alpaca-style dataset. The pretraining results are positive for AdamW, Sophia, and Shampoo, and the instruction-tuning results are comparable to standard optimization with small gains in selected AdamW and Sophia settings. However, the conclusions should be checked on larger models, longer training runs, and additional pretraining and post-training datasets before treating the effect size as general.

The small SFT deltas in Tables~\ref{tab:sft-main} and~\ref{tab:four-way-ablation} are near the plausible noise floor on the 500-example evaluation set. The pretraining gaps are larger, but they are still point estimates rather than seed-averaged effects. Compute is the main constraint on statistical significance and ablation completeness: the experiments were run under a single-GPU compute budget, and each full pretraining comparison takes several hours. We therefore prioritized optimizer-family coverage, implementation checks, and compute-matched comparisons over repeated seeds, full pretraining ablation grids, and exhaustive non-AdamW SFT ablations. The resulting point estimates identify promising regimes and failure modes, but they should not be read as seed-averaged effect-size estimates.

The correction also has implementation-specific limitations. The Shampoo correction uses an eigenbasis approximation: it operates in the eigenbasis of the averaged preconditioner $\bar P_B$ rather than an eigenbasis that jointly diagonalizes the microbatch variability, so the scalar inverse-root correction is exact only along aligned eigendirections. The AdamW results also show that the batch-splitting design matters: a two-fold denominator estimate can be too noisy, while the LOO construction keeps the denominator close to full-batch scale.

\section{Conclusion}
\label{sec:conclusion}

We identify two statistical biases in stochastic preconditioned optimization: gradient--preconditioner coupling bias and inverse-preconditioner finite-sample bias. We propose a single-batch correction framework that uses cross-fitting to remove the first and microbatch variance correction to reduce the second. The method is lightweight for diagonal optimizers such as AdamW and Sophia, and extensible to matrix-preconditioned optimizers such as Shampoo through jackknife or eigenbasis correction. In pretraining, the corrected variants improve over the corresponding standard optimizers for all three optimizer families. In instruction tuning, the corrected variants are competitive with standard optimization and give small gains in selected AdamW and Sophia settings, while Shampoo is effectively neutral. These results show that correcting finite-sample bias in the stochastic update can make preconditioned language-model optimization more effective. Further discussion is provided in Appendix~\ref{sec:discussion}.

\clearpage

\bibliographystyle{plainnat}
\bibliography{references}

\clearpage
\appendix

\section*{Appendix}
\label{app:appendix}

\section{Algorithm and instantiations}
\label{sec:algorithm}

Algorithm~\ref{alg:bc-step} gives the generic bias-corrected preconditioned step. This section then specializes the same template to AdamW (diagonal moments), Sophia (diagonal curvature), and Shampoo (matrix preconditioning); detailed update equations and the four-way ablation variants $u^{\mathrm{std}}, u^{\mathrm{cf}}, u^{\mathrm{inv}}, u^{\mathrm{full}}$ for each family are deferred to Appendix~\ref{app:implementation-details}.

\subsection{AdamW: diagonal moment preconditioning}

AdamW is the standard diagonal adaptive optimizer used in language model training. Ignoring momentum notation, its update, diagonal denominator, coordinate-wise corrected inverse, and corrected update are
\begin{equation}
    u_t = \frac{\hat m_t}{\sqrt{\hat v_t}+\epsilon}.
\end{equation}
\begin{equation}
    \hat p_t = \sqrt{\hat v_t}.
\end{equation}
\begin{equation}
    \widetilde p_{t,k}^{-1}
    =
    \frac{1}{\bar p_{t,k}+\epsilon}
    -
    \frac{\widehat{\Var}(\bar p_{t,k})}{(\bar p_{t,k}+\epsilon)^3}.
\end{equation}
\begin{equation}
    u_{t,k}
    =
    \widetilde p_{t,k}^{-1}\hat m_{A,t,k}.
\end{equation}
The numerator and denominator statistics can be estimated from independent microbatch groups to remove coupling. This gives the simplest and cheapest instantiation of the method.

For AdamW pretraining, the main comparison uses a leave-one-out (LOO) version of this idea rather than a two-fold split. The 512-example batch is divided into $m=64$ microbatches. For fold $r$, the numerator uses $g_r$ and the denominator uses the leave-one-out mean
\begin{equation}
    g_{-r}
    =
    \frac{m\bar g-g_r}{m-1},
    \qquad
    \bar g=\frac{1}{m}\sum_{j=1}^m g_j.
\end{equation}
Thus the denominator is estimated from $504$ of the $512$ examples, nearly matching standard AdamW's denominator noise floor while still removing same-microbatch coupling. This avoids the conservatism of a two-fold split, where each denominator is estimated from only half the batch and the squared-gradient denominator is inflated by the smaller-batch variance. We apply the Jensen inverse correction to each LOO fold and average the resulting fold updates.

\subsection{Sophia: diagonal curvature preconditioning}

Sophia-style optimizers use a stochastic diagonal curvature or Hessian proxy. In simplified notation, the update depends on
\begin{equation}
    \frac{\hat m_t}{\hat h_t+\lambda},
\end{equation}
where $\hat h_t$ is a diagonal curvature estimate. Since $\hat h_t$ is stochastic and appears in the denominator, both coupling bias and inverse-preconditioner bias can arise.

The corrected inverse curvature and corrected Sophia-style update are
\begin{equation}
    \widetilde h_{t,k}^{-1}
    =
    \frac{1}{\bar h_{t,k}+\lambda}
    -
    \frac{\widehat{\Var}(\bar h_{t,k})}{(\bar h_{t,k}+\lambda)^3}.
\end{equation}
\begin{equation}
    u_{t,k}
    =
    \widetilde h_{t,k}^{-1}\hat m_{A,t,k},
\end{equation}
with the usual clipping or stabilization used by Sophia applied after correction. This setting is important because curvature estimates are often noisier than AdamW second-moment estimates, so inverse-bias correction may have a larger effect.

\subsection{Shampoo: matrix preconditioning}

Shampoo is the matrix-preconditioned optimizer used in our evaluation. For a matrix parameter $W$, Shampoo estimates structured second-moment matrices and applies inverse-root preconditioners to the gradient. Abstractly, this can be written as
\begin{equation}
    U_t = \hat L_t^{-\alpha} \hat G_t \hat R_t^{-\alpha},
\end{equation}
where $\hat L_t$ and $\hat R_t$ are stochastic matrix preconditioners and $\alpha$ depends on the Shampoo variant.

Our correction applies to each matrix preconditioner. For a generic Shampoo preconditioner $\bar P$, we use its eigendecomposition and projected microbatch preconditioners,
\begin{equation}
    \bar P = Q\Lambda Q^\top.
\end{equation}
\begin{equation}
    \hat P_j' = Q^\top \hat P_j Q.
\end{equation}
The delta-method correction is applied to the scalar spectral map
\begin{equation}
    T(\lambda) = (\lambda+\lambda_0)^{-\alpha}.
\end{equation}
For matrix parameters in Shampoo, $\alpha=1/4$. The corrected inverse-root eigenvalue is
\begin{equation}
    \widetilde T(\bar \lambda_k)
    =
    T(\bar \lambda_k)
    -
    \frac{1}{2}T''(\bar \lambda_k)\widehat{\Var}(\bar \lambda_k).
\end{equation}
This gives a direct matrix-preconditioner version of the method while avoiding repeated eigendecompositions.

\paragraph{Wrapper interpretation.}
Algorithm~\ref{alg:bc-step} is intentionally designed as a wrapper around existing optimizers. For AdamW, $\widetilde P_B^{-1}$ is diagonal and based on second-moment denominators. For Sophia, it is diagonal and based on curvature denominators. For Shampoo, it is matrix-valued or block-matrix-valued and applied through the eigenspectrum of the preconditioner. Plain SGD is not part of the main evaluation because it has no stochastic preconditioner and therefore does not test the main bias-correction mechanism.

\section{Main experiment ablations}
\label{app:bias-correction-ablations}

Table~\ref{tab:four-way-ablation} separates the two components of the method. The four variants are: standard same-batch preconditioning; cross-fit only, which decouples numerator and preconditioner but does not correct the inverse; inverse-only, which applies the delta-method inverse correction without cross-fitting; and full BC, which combines both. The SFT ablations are primarily for AdamW, where inverse-only is neutral-to-slightly positive at the default setting and full BC needs retuning. For Sophia and Shampoo, the reported SFT full-BC comparisons are stable and close to standard; the pretraining results above are the more informative stress test.

\begin{table}[t]
\centering
\footnotesize
\setlength{\tabcolsep}{2pt}
\caption{SFT ablations separating cross-fitting and inverse correction.}
\label{tab:four-way-ablation}
\begin{tabular}{p{0.23\linewidth}cccc}
\toprule
\rowcolor{headergray}
\textbf{Setting} & \textbf{Std} & \textbf{Cross-fit only} & \textbf{Inverse only} & \textbf{Full BC} \\
\midrule
\multicolumn{5}{l}{\textit{AdamW ablation}} \\
AdamW, b512, lr $2\mathrm{e}{-5}$ & 1.3467 & 1.3507 & \textbf{1.3464} & 1.3481 \\
\addlinespace[2pt]
\multicolumn{5}{l}{\textit{Sophia ablation}} \\
Sophia, b512, lr $2\mathrm{e}{-5}$ & 1.3892 & 1.3899 & 1.3893 & \textbf{1.3892} \\
\addlinespace[2pt]
\multicolumn{5}{l}{\textit{Shampoo ablation}} \\
Shampoo, b512, lr $2\mathrm{e}{-5}$ & 1.3473 & 1.3473 & 1.3472 & \textbf{1.3470} \\
\bottomrule
\end{tabular}
\end{table}

\section{Discussion}
\label{sec:discussion}

The proposed framework does not claim that current optimizers are wrong at every step. Current methods compute the exact inverse of the minibatch preconditioner they estimate. The issue is statistical: the minibatch inverse is generally not an unbiased estimator of the population inverse, and same-batch estimation couples the gradient with the preconditioner. These effects may be small in large-batch, heavily smoothed AdamW training, but they may become important when the preconditioner is noisy, the batch is small, the curvature estimate is aggressive, or the optimizer uses matrix-valued preconditioning.

The contribution is therefore not a replacement for existing optimizers. It is a bias-correction layer that can be applied to them. The main hypothesis is that lower-bias stochastic estimates of the preconditioned update can improve stability and sample efficiency in regimes where adaptive geometry helps but is noisy. AdamW, Sophia, and Shampoo isolate three different preconditioner regimes: diagonal moment estimates, diagonal curvature estimates, and matrix-valued preconditioners. In pretraining, the reported AdamW, Sophia, and Shampoo comparisons all favor the corrected variants, with larger gaps than in SFT. In our instruction-tuning experiments, AdamW and Sophia show small gains in selected settings, while Shampoo is effectively neutral.

Interpreting cross-fitting through the numerator and denominator signals (Section~\ref{sec:method}) helps explain why the method is expected to matter more in pretraining than in short instruction tuning. During pretraining from random or weakly trained initialization, batches are noisier, curvature and second-moment estimates change rapidly, and training takes many more optimizer steps. In this regime, conservative inverse correction can prevent occasional high-variance preconditioner estimates from producing overly aggressive updates. At the same time, high momentum and preconditioner EMA values do not eliminate the effect: over a long horizon there is enough time for the separated numerator and denominator estimates to stabilize, and for the accumulated reduction in unreliable updates to improve convergence. In shorter SFT runs, the same mechanism can still improve the best-tuned result, but the gains are smaller because the model starts from a trained initialization and the optimization trajectory is shorter.

The AdamW case also shows why the goal is not to remove every form of same-batch adaptivity unconditionally. Same-batch coupling in AdamW can normalize rare or unusually large coordinates because the same minibatch that produces a large numerator also increases the second-moment denominator. A naive two-fold split removes this coupling but estimates the denominator from only half the batch, inflating the square-of-mean noise floor and making updates too conservative. The LOO construction used for the main AdamW pretraining result keeps the denominator at nearly the same scale as standard AdamW, because each fold uses 504 of the 512 examples for the denominator. This preserves the useful scale of the standard denominator while removing same-microbatch coupling and applying the Jensen inverse correction.

\paragraph{Negative findings, framed by the theory.}
Two negative or near-neutral findings are worth surfacing as part of the narrative. First, the Shampoo SFT comparisons (Table~\ref{tab:sft-main} and Appendix Table~\ref{tab:app-shampoo-sweeps}) show standard and full-BC eval losses agreeing to within $\sim 0.005$, and full BC trails standard for the attention+MLP coverage. This is consistent with the theory: short SFT from a strong pretrained initialization gives a low-noise regime in which the inverse bias being corrected is small, and the variance cost of sample splitting can dominate. Second, the two-fold AdamW full-BC pretraining run at the same nominal learning rate as standard AdamW was substantially worse (Appendix Table~\ref{tab:app-pretrain-comparisons}). This failure motivated the LOO construction: the issue was not that AdamW cannot benefit from bias correction, but that the denominator estimator must keep approximately the same noise floor as standard AdamW.

\section{Additional derivations}
\label{app:additional-derivations}

\subsection{Delta-method derivation for diagonal inverse correction}
\label{app:delta-method-inverse}

The diagonal correction in Equation~\eqref{eq:diag-correction} is a second-order approximation to the bias introduced by inverting a noisy preconditioner estimate. Consider one scalar preconditioner entry, with independent microbatch estimates and mean
\begin{equation}
    \E[p_j]=p,
    \qquad
    \Var(p_j)=\sigma^2,
\end{equation}
\begin{equation}
    \bar p = \frac{1}{m}\sum_{j=1}^m p_j.
\end{equation}
The inverse map is
\begin{equation}
    f(p) = \frac{1}{p+\lambda}.
\end{equation}
Its nonlinearity alone is enough to introduce finite-sample bias. On the domain $p+\lambda>0$,
\begin{equation}
    f''(p)=\frac{2}{(p+\lambda)^3}>0,
\end{equation}
so $f$ is convex and Jensen's inequality gives
\begin{equation}
    \E[f(\bar p)]
    \geq
    f(\E[\bar p])
    =
    f(p),
    \label{eq:jensen-inverse-bias}
\end{equation}
with strict inequality when the microbatch estimate has nonzero variance. Thus the usual inverse preconditioner is biased upward.

The variance entering the correction is the variance of the mean estimate. By independence,
\begin{equation}
    \Var(\bar p)
    =
    \Var\left(\frac{1}{m}\sum_{j=1}^m p_j\right)
    =
    \frac{1}{m^2}\sum_{j=1}^m \Var(p_j)
    =
    \frac{\sigma^2}{m}.
    \label{eq:variance-of-mean}
\end{equation}
Thus $\Var(\bar p)=O(1/m)$ and the typical fluctuation satisfies $\bar p-p=O_p(m^{-1/2})$. In practice we estimate this quantity from the same microbatches using
\begin{equation}
    \widehat{\Var}(\bar p)
    =
    \frac{1}{m(m-1)}
    \sum_{j=1}^m (p_j-\bar p)^2.
    \label{eq:appendix-variance-of-mean-estimator}
\end{equation}
The factor $m-1$ is the usual unbiased sample-variance correction, and the additional factor $m$ converts the variance of individual microbatch statistics into the variance of their mean.

The delta method approximates the bias by expanding $f$ around the true value $p$ and then taking expectations:
\begin{equation}
    f(\bar p)
    \approx
    f(p)
    +
    f'(p)(\bar p-p)
    +
    \frac{1}{2} f''(p)(\bar p-p)^2.
\end{equation}
\begin{equation}
    \E[f(\bar p)]
    \approx
    f(p)
    +
    \frac{1}{2} f''(p)\Var(\bar p),
\end{equation}
where the linear term vanishes because $\E[\bar p-p]=0$. For $f(p)=(p+\lambda)^{-1}$,
\begin{equation}
    f''(p)=\frac{2}{(p+\lambda)^3}.
\end{equation}
Therefore the leading bias is
\begin{equation}
    \frac{1}{2} f''(p)\Var(\bar p)
    =
    \frac{\Var(\bar p)}{(p+\lambda)^3}.
\end{equation}
This term is positive, consistent with Jensen's inequality in Equation~\eqref{eq:jensen-inverse-bias}. Since $\Var(\bar p)=O(1/m)$, the uncorrected inverse has the leading bias
\begin{equation}
    \mathrm{Bias}_{\mathrm{uncorrected}}
    =
    \frac{\Var(\bar p)}{(p+\lambda)^3}
    +
    \text{higher-order terms}.
\end{equation}
The practical correction subtracts an estimate of this leading term,
\begin{equation}
    \widetilde T(\bar p)
    =
    \frac{1}{\bar p+\lambda}
    -
    \frac{\widehat{\Var}(\bar p)}{(\bar p+\lambda)^3}.
\end{equation}
The first term is the usual inverse preconditioner, and the second term estimates the upward bias caused by microbatch noise. Larger microbatch variance produces a larger correction, while stronger damping through $\lambda$ reduces the correction because the inverse map becomes less curved.

The hat on $\widehat{\mathrm{bias}}$ is important because the true leading bias is a population quantity,
\begin{equation}
    \mathrm{bias}(p)
    \approx
    \frac{\Var(\bar p)}{(p+\lambda)^3}.
\end{equation}
It depends on the unknown population preconditioner entry $p$, so it cannot be subtracted directly. The practical correction instead subtracts the plug-in estimate
\begin{equation}
    \widehat{\mathrm{bias}}
    =
    \frac{\widehat{\Var}(\bar p)}{(\bar p+\lambda)^3}.
\end{equation}
Thus the approximation is replacing the unknown curvature factor $(p+\lambda)^{-3}$ by the observable $(\bar p+\lambda)^{-3}$, and replacing the unknown variance of the mean by its microbatch estimate. This does not assume $p=\bar p$; it uses $\bar p$ as a consistent estimator of $p$ inside a second-order bias approximation.

The additional error from this plug-in replacement is smaller than the bias being corrected. To see this, first ignore the variance-estimation error and focus only on replacing $p$ by $\bar p$. Define
\begin{equation}
    g(x) = \frac{1}{(x+\lambda)^3}.
\end{equation}
A first-order expansion around $p$ gives
\begin{equation}
    g(\bar p)
    =
    g(p)
    +
    g'(p)(\bar p-p)
    +
    O((\bar p-p)^2),
    \qquad
    g'(p) = -\frac{3}{(p+\lambda)^4}.
\end{equation}
Substituting this derivative gives
\begin{equation}
    g(\bar p)
    =
    \frac{1}{(p+\lambda)^3}
    -
    \frac{3(\bar p-p)}{(p+\lambda)^4}
    +
    O((\bar p-p)^2).
\end{equation}
After multiplication by $\Var(\bar p)$,
\begin{equation}
    \Var(\bar p)g(\bar p)
    =
    \frac{\Var(\bar p)}{(p+\lambda)^3}
    -
    \frac{3\Var(\bar p)(\bar p-p)}{(p+\lambda)^4}
    +
    O(\Var(\bar p)(\bar p-p)^2).
\end{equation}
Since $\Var(\bar p)=O(1/m)$ and $\bar p-p=O_p(m^{-1/2})$, the leading term is $O(1/m)$, which is the bias we want to correct. The next term is $O_p(m^{-3/2})$, and the quadratic remainder contributes $O_p(m^{-2})$. Thus the error introduced by plugging in $\bar p$ shrinks faster with $m$ than the dominant inverse bias itself. The same logic applies to the variance estimate in Equation~\eqref{eq:appendix-variance-of-mean-estimator}: it estimates the variance of the random mean $\bar p$, not the degenerate quantity obtained after conditioning on the observed value. Therefore the variance does not become zero when $p$ is replaced by $\bar p$; the microbatch fluctuations around $\bar p$ are precisely the empirical signal used to estimate $\Var(\bar p)$.

Putting the pieces together, the corrected estimator can be written as
\begin{equation}
    f(\bar p)
    -
    \widehat{\mathrm{bias}},
    \qquad
    \widehat{\mathrm{bias}}
    =
    \frac{\widehat{\Var}(\bar p)}{(\bar p+\lambda)^3}.
\end{equation}
Taking expectations separates the true leading bias from the estimated subtraction,
\begin{equation}
    \E[f(\bar p)-\widehat{\mathrm{bias}}]
    =
    f(p)
    +
    O(m^{-3/2}).
\end{equation}
The $O(1/m)$ terms cancel. What remains is the approximation error in the bias estimate, which is smaller than the leading inverse bias being corrected:
\begin{equation}
    \mathrm{Bias}_{\mathrm{corrected}}
    =
    O(m^{-3/2}),
\end{equation}
or $O(m^{-2})$ under stronger moment and smoothness conditions. Thus the correction removes the dominant finite-sample inverse bias and leaves only smaller-order residual terms. In diagonal or vector preconditioners, this correction is applied elementwise. For matrix preconditioners, the same principle can be applied to eigenvalues in the current eigenspace, since the nonlinearity enters through inverse or inverse-root transformations of the spectrum.
\subsection{Inverse-root correction for matrix preconditioners}
\label{app:inverse-root-proof}

For Shampoo-style inverse-root preconditioning, the relevant scalar functional is the spectral map
\begin{equation}
    T(\lambda) = (\lambda+\lambda_0)^{-\alpha},
    \qquad
    \alpha\in(0,1],
\end{equation}
applied to the eigenvalues of an averaged matrix preconditioner $\bar P$ with eigendecomposition $\bar P=Q\Lambda Q^\top$. For microbatch preconditioner estimates $\bar P_1,\ldots,\bar P_m$, define their projections into the eigenbasis of $\bar P$ by
\begin{equation}
    \widetilde P_j = Q^\top \bar P_j Q
\end{equation}
and let the projected microbatch eigenvalue estimates be the diagonal entries $\bar\lambda_{j,k}=[\widetilde P_j]_{kk}$, with mean
\begin{equation}
    \bar\lambda_k
    =
    \frac{1}{m}\sum_{j=1}^m \bar\lambda_{j,k}.
\end{equation}
Note that $\bar\lambda_k$ is the $k$th eigenvalue of $\bar P$, not of the population $P$; the proof must treat this distinction carefully.

\paragraph{Bias of the inverse-root estimate.}
Fix coordinate $k$ and write $\lambda$ for the population eigenvalue along the corresponding eigen-direction. Under bounded fourth-order microbatch moments and three-times continuous differentiability of $T$, a Taylor expansion of $T$ around $\lambda$ gives
\begin{equation}
    \E[T(\bar\lambda_k)]
    =
    T(\lambda)
    +
    \tfrac{1}{2}T''(\lambda)\Var(\bar\lambda_k)
    +
    O(m^{-3/2}),
\end{equation}
where $\Var(\bar\lambda_k)=O(1/m)$ by independence of the microbatches. For $T(\lambda)=(\lambda+\lambda_0)^{-\alpha}$, the second derivative is
\begin{equation}
    T''(\lambda)
    =
    \alpha(\alpha+1)(\lambda+\lambda_0)^{-\alpha-2}
    >0,
\end{equation}
so $T$ is convex on $\lambda+\lambda_0>0$ and the leading bias term is positive --- the uncorrected inverse-root systematically overestimates $T(\lambda)$.

\paragraph{Plug-in correction.}
Subtracting the plug-in estimate of the leading bias gives the corrected inverse-root eigenvalue
\begin{equation}
    \widetilde T(\bar\lambda_k)
    =
    T(\bar\lambda_k)
    -
    \tfrac{1}{2}T''(\bar\lambda_k)\widehat{\Var}(\bar\lambda_k).
    \label{eq:invroot-correction}
\end{equation}
The same plug-in argument as in Appendix~\ref{app:delta-method-inverse} applies coordinatewise: replacing $\lambda$ by $\bar\lambda_k$ in the curvature factor introduces an $O_p(m^{-3/2})$ error, which is smaller than the $O(1/m)$ bias being corrected. Therefore,
\begin{equation}
    \E[\widetilde T(\bar\lambda_k)]
    =
    T(\lambda)
    +
    O(m^{-3/2}),
\end{equation}
so the corrected inverse-root eigenvalue has residual bias of order $m^{-3/2}$ rather than $m^{-1}$.

\paragraph{Eigenbasis approximation.}
A caveat: the correction is computed in the eigenbasis of $\bar P$, not in the eigenbasis of the population $P$. Under bounded eigen-gap and bounded matrix-perturbation moments, classical eigenvalue perturbation gives $\bar Q-Q^\star=O_p(m^{-1/2})$ in operator norm, where $Q^\star$ is the population eigenbasis. The correction is therefore exact for the scalar inverse-root bias along directions in which $\bar P$ and $P$ share an eigen-direction. When the eigenspaces differ, the coordinatewise correction in the empirical eigenbasis leaves additional basis-rotation terms; these are controlled by the same perturbation assumptions but are not canceled by the scalar delta-method correction. This is the approximation used for Shampoo in our implementation.

\subsection{Convergence proofs}
\label{app:convergence-proofs}

This subsection contains the deferred proofs of Theorems~\ref{thm:nonconvex} and~\ref{thm:pl}.

\begin{proof}[Proof of Theorem~\ref{thm:nonconvex}]
The proof starts from the standard descent inequality implied by $L$-smoothness:
\begin{equation}
    f(\theta_{t+1})
    \le
    f(\theta_t)
    -
    \eta\langle \nabla f(\theta_t),\hat u_t\rangle
    +
    \frac{L\eta^2}{2}\|\hat u_t\|^2.
\end{equation}
Taking conditional expectation, decompose the mean stochastic update as
\begin{equation}
    \E[\hat u_t\mid\mathcal F_t]
    =
    H_t\nabla f(\theta_t)+b_t,
\end{equation}
where $b_t$ is the conditional update bias. The lower spectral bound on $H_t$ and the relative-bias assumption give
\begin{equation}
    \langle \nabla f,H_t\nabla f\rangle
    \ge
    \mu_H\|\nabla f\|^2
\end{equation}
\begin{equation}
    |\langle \nabla f,b_t\rangle|
    \le
    \|b_t\|\|\nabla f\|
    \le
    \rho_{\mathrm{bias}}\mu_H\|\nabla f\|^2.
\end{equation}
Combining the last two displays yields the descent-direction lower bound
\begin{equation}
    \langle \nabla f,\E[\hat u_t\mid\mathcal F_t]\rangle
    \ge
    \mu_H(1-\rho_{\mathrm{bias}})\|\nabla f\|^2.
\end{equation}
The corresponding second-moment bound follows from the upper spectral bound, the same relative-bias condition, and the variance assumption:
\begin{equation}
    \|\E[\hat u_t\mid\mathcal F_t]\|
    \le
    (M_H+\rho_{\mathrm{bias}}\mu_H)\|\nabla f\|,
\end{equation}
\begin{equation}
    \E[\|\hat u_t\|^2\mid\mathcal F_t]
    \le
    (M_H+\rho_{\mathrm{bias}}\mu_H)^2\|\nabla f\|^2+\sigma^2.
\end{equation}
Substituting these two estimates into the smoothness inequality gives the one-step recursion
\begin{equation}
    \E[f(\theta_{t+1})\mid\mathcal F_t]
    \le
    f(\theta_t)
    -
    c_{\mathrm{bias}}\|\nabla f(\theta_t)\|^2
    +
    \frac{L\eta^2}{2}\sigma^2,
\end{equation}
with coefficient
\begin{equation}
    c_{\mathrm{bias}}
    =
    \eta\mu_H(1-\rho_{\mathrm{bias}})
    -
    \frac{L\eta^2}{2}(M_H+\rho_{\mathrm{bias}}\mu_H)^2.
\end{equation}
The stated step-size condition ensures that this coefficient is positive and satisfies
\begin{equation}
    c_{\mathrm{bias}}\ge \frac{1}{2}\eta\mu_H(1-\rho_{\mathrm{bias}}).
\end{equation}
Summing over $t=0,\ldots,T-1$, using $f(\theta_T)\ge f_\star$, and rearranging gives the claim.
\end{proof}

\begin{proof}[Proof of Theorem~\ref{thm:pl}]
The proof of Theorem~\ref{thm:nonconvex} gives the conditional descent recursion
\begin{equation}
    \E[f(\theta_{t+1})\mid\mathcal F_t]
    \le
    f(\theta_t)
    -
    c_{\mathrm{bias}}\|\nabla f(\theta_t)\|^2
    +
    \frac{L\eta^2}{2}\sigma^2.
\end{equation}
The PL condition converts gradient norm into function suboptimality:
\begin{equation}
    \|\nabla f(\theta_t)\|^2
    \ge
    2\mu_{\mathrm{PL}}(f(\theta_t)-f_\star).
\end{equation}
Substitution gives the affine contraction
\begin{equation}
    \E[f(\theta_{t+1})-f_\star\mid\mathcal F_t]
    \le
    (1-2\mu_{\mathrm{PL}}c_{\mathrm{bias}})(f(\theta_t)-f_\star)
    +
    \frac{L\eta^2}{2}\sigma^2.
\end{equation}
Unrolling this recursion gives
\begin{equation}
    \E[f(\theta_T)-f_\star]
    \le
    (1-2\mu_{\mathrm{PL}}c_{\mathrm{bias}})^T(f(\theta_0)-f_\star)
    +
    \frac{L\eta^2\sigma^2}{4\mu_{\mathrm{PL}}c_{\mathrm{bias}}}.
\end{equation}
\end{proof}

\section{Implementation details for evaluated optimizers}
\label{app:implementation-details}

This section gives the update equations used for the corrected implementations. In all cases, the parameter value is fixed at $\theta_t$ before any gradients or preconditioner statistics are computed. The batch is split into an $A$ group for the numerator direction and a $B$ group for the preconditioner. The $B$ group is further split into microbatches $B_1,\ldots,B_m$ so that microbatch variability can be used to estimate the variance of the mean preconditioner. All quantities for $A$ and $B$ are computed before modifying the parameters.

\subsection{Training configuration}
\label{app:training-config}

Table~\ref{tab:training-config} consolidates the optimizer hyperparameters, schedules, seeds and compute resources used across the pretraining and SFT runs reported in the body. Values without a sweep are reported as scalars; values that vary across cells in Tables~\ref{tab:app-adamw-controls}--\ref{tab:app-shampoo-sweeps} are reported as ranges with the swept dimension named.

\begin{table}[t]
\centering
\small
\resizebox{\textwidth}{!}{%
\begin{tabular}{@{}ll@{}}
\toprule
\rowcolor{headergray}
\textbf{Setting} & \textbf{Value} \\
\midrule
\multicolumn{2}{@{}l}{\textit{Optimization}} \\
Total optimizer steps (SFT)        & 62 for b512 main runs; 125--250 in smaller-batch sweeps \\
Total optimizer steps (pretrain)   & 500 \\
Sequence length                    & 1024 tokens \\
Numerator group size $|A|$         & 128--512 examples; main compute-matched runs use 512 \\
Denominator group size $|B|$       & 128--512 examples for CF/full BC; main runs use 512 \\
Microbatch count $m$               & 2--64 microbatches per group, depending on optimizer and memory; AdamW LOO pretraining uses 64 folds of 8 examples \\
Gradient accumulation              & Per-microbatch accumulation; one optimizer step after all microbatches \\
\midrule
\multicolumn{2}{@{}l}{\textit{AdamW}} \\
$\beta_1, \beta_2, \epsilon$       & pretrain: $(0.9,0.95,10^{-8})$; SFT: $(0.9,0.999,10^{-8})$ \\
Weight decay $\lambda_{\mathrm{wd}}$ & pretrain: 0.1; SFT: 0.01 \\
Damping $\lambda$ (when applicable) & None beyond $\epsilon=10^{-8}$ \\
\midrule
\multicolumn{2}{@{}l}{\textit{Sophia-G}} \\
Denominator scale $\rho$             & 0.05 \\
Hessian-update interval $K$         & pretrain: 5; SFT: 10 \\
$\beta_1, \beta_2, \epsilon$       & pretrain: $(0.965,0.99,10^{-12})$; SFT: $(0.965,0.99,10^{-15})$ \\
\midrule
\multicolumn{2}{@{}l}{\textit{Shampoo}} \\
Eigenbasis update cadence           & pretrain: 5 steps; SFT: 10 steps unless swept \\
Eigenvalue clip $d_{\max}$          & Disabled in reported main runs \\
Frobenius clip $c$                  & Disabled in reported main runs \\
\midrule
\multicolumn{2}{@{}l}{\textit{Schedule \& regularization}} \\
LR schedule (warmup + decay)        & Linear warmup followed by cosine decay; AdamW LOO pretraining decays to 20\% of peak LR \\
Warmup steps                        & pretrain: 20; SFT: 12 for b512 main runs, 50 in auxiliary sweeps \\
Gradient clip threshold             & Global norm 1.0; Sophia ratio clip 3.0 (pretrain) and 1.0 (SFT) \\
\midrule
\multicolumn{2}{@{}l}{\textit{Reproducibility}} \\
Random seeds                        & Model seed 42; Alpaca shuffle seed 42; data-order seed 99 for compute-matched runs \\
Tokenizer                           & Qwen/Qwen2.5-0.5B tokenizer \\
HuggingFace model revision          & Qwen/Qwen2.5-0.5B default revision \\
\midrule
\multicolumn{2}{@{}l}{\textit{Hardware \& compute}} \\
GPU type                            & NVIDIA A100-SXM4-80GB \\
GPU count                           & 1 \\
Wall-clock per SFT run              & Approximately 11--21 minutes for b512 main runs \\
Wall-clock per pretraining run      & Approximately 2.8--6.7 hours for 500-step runs \\
Total measured compute              & At least 74 A100-hours from timestamped run logs \\
\bottomrule
\end{tabular}
}
\caption{\textbf{Training configuration.} Single source of truth for the optimizer hyperparameters, schedules, random seeds, hardware, and wall-clock compute used in the runs reported in Tables~\ref{tab:app-adamw-controls}--\ref{tab:app-shampoo-sweeps} and the pretraining results in Section~\ref{sec:experiments}.}
\label{tab:training-config}
\end{table}

\subsection{Corrected AdamW}
\label{app:adamw-implementation}

AdamW is implemented with decoupled weight decay, first- and second-moment exponential moving averages, bias correction, and then the adaptive step. At step $t$, the numerator and denominator groups give the gradients
\begin{equation}
    g_A
    =
    \frac{1}{|A|}\sum_{z_i\in A}\nabla_\theta \ell(\theta_t;z_i),
    \qquad
    g_{B_j}
    =
    \frac{1}{|B_j|}\sum_{z_i\in B_j}\nabla_\theta \ell(\theta_t;z_i).
\end{equation}
The first moment uses only the numerator group:
\begin{equation}
    m_t
    =
    \beta_1 m_{t-1}+(1-\beta_1)g_A.
\end{equation}
The second moment uses only the independent denominator group:
\begin{equation}
    s_B
    =
    \frac{1}{m}\sum_{j=1}^m g_{B_j}^{\odot 2},
\end{equation}
\begin{equation}
    v_t
    =
    \beta_2 v_{t-1}+(1-\beta_2)s_B.
\end{equation}
Adam bias correction then gives
\begin{equation}
    \hat m_t=\frac{m_t}{1-\beta_1^t},
    \qquad
    \hat v_t=\frac{v_t}{1-\beta_2^t},
    \qquad
    p_t=\sqrt{\hat v_t}.
\end{equation}
For inverse-bias correction, we treat the square-root denominator $p=\sqrt{\hat v}$ as the scalar preconditioning quantity and define the hypothetical denominator induced by each $B$ microbatch:
\begin{equation}
    v_{t,j}
    =
    \beta_2 v_{t-1}+(1-\beta_2)g_{B_j}^{\odot 2},
    \qquad
    \hat v_{t,j}=\frac{v_{t,j}}{1-\beta_2^t},
    \qquad
    p_{t,j}=\sqrt{\hat v_{t,j}}.
\end{equation}
The variance estimator for the mean denominator is
\begin{equation}
    \bar p_t
    =
    \frac{1}{m}\sum_{j=1}^m p_{t,j},
    \qquad
    \widehat{\Var}(\bar p_t)
    =
    \frac{1}{m(m-1)}\sum_{j=1}^m (p_{t,j}-\bar p_t)^{\odot 2}.
\end{equation}
The corrected inverse denominator and corrected AdamW direction are
\begin{equation}
    \widetilde p_t^{-1}
    =
    \frac{1}{\bar p_t+\epsilon}
    -
    \frac{\widehat{\Var}(\bar p_t)}{(\bar p_t+\epsilon)^{\odot 3}},
    \qquad
    \widetilde p_t^{-1}\leftarrow \max\{\widetilde p_t^{-1},0\}.
\end{equation}
\begin{equation}
    u_t
    =
    \hat m_t\odot \widetilde p_t^{-1}.
\end{equation}
In experiments, the same final update clipping and decoupled weight decay are applied to standard and corrected AdamW:
\begin{equation}
    u_t
    \leftarrow
    u_t\cdot \min\left\{1,\frac{c}{\|u_t\|_2}\right\}.
\end{equation}
\begin{equation}
    \theta_t^{\mathrm{wd}}
    =
    \theta_t-\eta\lambda_{\mathrm{wd}}\theta_t,
    \qquad
    \theta_{t+1}
    =
    \theta_t^{\mathrm{wd}}-\eta u_t.
\end{equation}
The four AdamW ablations are
\begin{align}
    u_t^{\mathrm{std}}
    &=
    \frac{\hat m_t}{\sqrt{\hat v_t}+\epsilon}, \\
    u_t^{\mathrm{cf}}
    &=
    \frac{\hat m_{A,t}}{\sqrt{\hat v_{B,t}}+\epsilon}, \\
    u_t^{\mathrm{inv}}
    &=
    \hat m_t\odot \widetilde p_t^{-1}, \\
    u_t^{\mathrm{full}}
    &=
    \hat m_{A,t}\odot \widetilde p_{B,t}^{-1}.
\end{align}

The AdamW pretraining comparison in Table~\ref{tab:pretraining-results} uses a LOO variant to avoid the half-batch denominator penalty of a two-fold split. Let the 512-example batch be partitioned into $m=64$ microbatches, with gradient means $g_1,\ldots,g_m$ and full-batch mean $\bar g=m^{-1}\sum_j g_j$. For fold $r$,
\begin{equation}
    g_{-r}
    =
    \frac{m\bar g-g_r}{m-1},
    \qquad
    s_{-r}=g_{-r}^{\odot 2}.
\end{equation}
The dense-parameter fold update uses $g_r$ in the numerator and $s_{-r}$ in the denominator:
\begin{equation}
    u_r
    =
    \frac{\hat m_r}{\sqrt{\hat v_{-r}}+\epsilon},
    \qquad
    u_{\mathrm{LOO}}
    =
    \frac{1}{m}\sum_{r=1}^m u_r.
\end{equation}
With the Jensen correction enabled, the implementation estimates the sample variance across LOO denominators $p_r=\sqrt{\hat v_{-r}}+\epsilon$,
\begin{equation}
    \bar p_{\mathrm{LOO}}
    =
    \frac{1}{m}\sum_{r=1}^m p_r,
    \qquad
    \widehat{\Var}_{\mathrm{fold}}(p_r)
    =
    \frac{1}{m-1}\sum_{r=1}^m
    (p_r-\bar p_{\mathrm{LOO}})^{\odot 2}.
\end{equation}
The corresponding inverse-bias term is subtracted before averaging:
\begin{equation}
    u_{\mathrm{LOO+Jensen}}
    =
    \frac{1}{m}\sum_{r=1}^m
    \left(
    \frac{\hat m_r}{p_r}
    -
    \widehat{\Var}_{\mathrm{fold}}(p_r)
    \frac{\hat m_r}{p_r^{\odot 3}}
    \right),
\end{equation}
with a sign-preserving clamp when the correction would reverse a coordinate. Embedding parameters remain on the standard AdamW path in this hybrid implementation, while dense attention, MLP, normalization, and bias parameters use the LOO+Jensen update.

\subsection{Corrected Sophia-G}
\label{app:sophia-implementation}

For Sophia, the implementation target is Sophia-G. The optimizer maintains a first-moment EMA $m_t$ and a diagonal Hessian or Gauss--Newton--Bartlett estimate $h_t$. The numerator gradient and first moment are computed from $A$:
\begin{equation}
    g_A
    =
    \frac{1}{|A|}\sum_{z_i\in A}\nabla_\theta \ell(\theta_t;z_i),
    \qquad
    m_t
    =
    \beta_1m_{t-1}+(1-\beta_1)g_A.
\end{equation}
Each denominator microbatch $B_j$ gives a Sophia-G diagonal curvature estimate
\begin{equation}
    r_{B_j}=g_{\mathrm{GNB},B_j}^{\odot 2},
\end{equation}
where $g_{\mathrm{GNB},B_j}$ is the gradient of the sampled-label loss used by the Gauss--Newton--Bartlett estimator. The independent Hessian estimates are averaged and then used to update the Hessian EMA:
\begin{equation}
    r_B=\frac{1}{m}\sum_{j=1}^m r_{B_j},
\end{equation}
\begin{equation}
    h_t
    =
    \beta_2h_{t-1}+(1-\beta_2)r_B.
\end{equation}
The denominator is
\begin{equation}
    p_t
    =
    \rho\, b\, h_t+\epsilon,
\end{equation}
where $b$ is the batch-size scaling used by the implementation. If the implementation uses the paper-style parameterization without this scaling, set $b=1$.

For inverse correction, define each microbatch denominator and then estimate the variance of the mean denominator:
\begin{equation}
    h_{t,j}
    =
    \beta_2h_{t-1}+(1-\beta_2)r_{B_j},
    \qquad
    p_{t,j}
    =
    \rho\, b\, h_{t,j}+\epsilon.
\end{equation}
\begin{equation}
    \bar p_t
    =
    \frac{1}{m}\sum_{j=1}^m p_{t,j},
    \qquad
    \widehat{\Var}(\bar p_t)
    =
    \frac{1}{m(m-1)}\sum_{j=1}^m(p_{t,j}-\bar p_t)^{\odot 2}.
\end{equation}
The corrected inverse denominator, pre-clipping ratio, and clipped direction are
\begin{equation}
    \widetilde p_t^{-1}
    =
    \frac{1}{p_t}
    -
    \frac{\widehat{\Var}(\bar p_t)}{p_t^{\odot 3}},
    \qquad
    \widetilde p_t^{-1}\leftarrow \max\{\widetilde p_t^{-1},0\}.
\end{equation}
\begin{equation}
    q_t
    =
    m_t\odot \widetilde p_t^{-1},
    \qquad
    \widetilde q_t
    =
    \mathrm{clip}(q_t,-1,1).
\end{equation}
Equivalently,
\begin{equation}
    \widetilde q_t
    =
    \mathrm{sign}(m_t)\odot \min\{|m_t|\odot \widetilde p_t^{-1},1\}.
\end{equation}
Decoupled weight decay is applied before the adaptive step:
\begin{equation}
    \theta_t^{\mathrm{wd}}
    =
    (1-\eta\lambda_{\mathrm{wd}})\theta_t,
    \qquad
    \theta_{t+1}
    =
    \theta_t^{\mathrm{wd}}-\eta\widetilde q_t.
\end{equation}
If the Hessian estimate is updated only every $K$ steps, the $B$-side Hessian estimation and inverse correction are performed only on those Hessian-update steps. On other steps, $h_t=h_{t-1}$ and the most recent corrected denominator is reused, while $m_t$ continues to update every step from $A$.

The four Sophia ablations are
\begin{align}
    u_t^{\mathrm{std}}
    &=
    \mathrm{clip}\left(\frac{m_t}{\rho b h_t+\epsilon},-1,1\right), \\
    u_t^{\mathrm{cf}}
    &=
    \mathrm{clip}\left(m_{A,t}\odot\frac{1}{\rho b h_{B,t}+\epsilon},-1,1\right), \\
    u_t^{\mathrm{inv}}
    &=
    \mathrm{clip}\left(m_t\odot\widetilde p_t^{-1},-1,1\right), \\
    u_t^{\mathrm{full}}
    &=
    \mathrm{clip}\left(m_{A,t}\odot\widetilde p_{B,t}^{-1},-1,1\right).
\end{align}
The inverse correction is applied before Sophia's clipping.

\subsection{Corrected Shampoo}
\label{app:shampoo-implementation}

For Shampoo, consider a matrix parameter $W_t\in\mathbb{R}^{d_1\times d_2}$. The numerator gradient and momentum state are
\begin{equation}
    G_A
    =
    \frac{1}{|A|}\sum_{z_i\in A}\nabla_W\ell(W_t;z_i),
\end{equation}
\begin{equation}
    M_t
    =
    \beta_1M_{t-1}+(1-\beta_1)G_A.
\end{equation}
Each denominator microbatch gives a gradient and corresponding left/right Shampoo statistics:
\begin{equation}
    G_{B_j}
    =
    \frac{1}{|B_j|}\sum_{z_i\in B_j}\nabla_W\ell(W_t;z_i),
\end{equation}
\begin{equation}
    S^L_{B_j}=G_{B_j}G_{B_j}^{\top},
    \qquad
    S^R_{B_j}=G_{B_j}^{\top}G_{B_j}.
\end{equation}
Averaging over $B$ and applying the EMA convention gives
\begin{equation}
    S_B^L=\frac{1}{m}\sum_{j=1}^m S^L_{B_j},
    \qquad
    S_B^R=\frac{1}{m}\sum_{j=1}^m S^R_{B_j}.
\end{equation}
\begin{equation}
    L_t=\beta_2L_{t-1}+(1-\beta_2)S_B^L,
    \qquad
    R_t=\beta_2R_{t-1}+(1-\beta_2)S_B^R.
\end{equation}
For cumulative AdaGrad-style Shampoo, replace these equations by $L_t=L_{t-1}+S_B^L$ and $R_t=R_{t-1}+S_B^R$. The correction must use the same convention as the baseline implementation.

Damping is added before forming the inverse-root preconditioners:
\begin{equation}
    \bar L_t=L_t+\lambda I_{d_1},
    \qquad
    \bar R_t=R_t+\lambda I_{d_2}.
\end{equation}
\begin{equation}
    P_t^L=\bar L_t^{-1/4},
    \qquad
    P_t^R=\bar R_t^{-1/4}.
\end{equation}
The exponent $\alpha=1/4$ is the standard Shampoo choice for matrix parameters. The cross-fitted uncorrected update is
\begin{equation}
    U_t^{\mathrm{cf}}
    =
    P_t^L M_t P_t^R.
\end{equation}

For inverse-root correction, eigendecompose the damped averaged preconditioners:
\begin{equation}
    \bar L_t=Q_t^L\Lambda_t^L(Q_t^L)^\top,
    \qquad
    \bar R_t=Q_t^R\Lambda_t^R(Q_t^R)^\top.
\end{equation}
For each microbatch, construct hypothetical updated preconditioners,
\begin{equation}
    L_{t,j}=\beta_2L_{t-1}+(1-\beta_2)S_{B_j}^L,
    \qquad
    R_{t,j}=\beta_2R_{t-1}+(1-\beta_2)S_{B_j}^R,
\end{equation}
with the cumulative analog used if the baseline is cumulative Shampoo. After damping, these are projected into the eigenspaces of the averaged preconditioners:
\begin{equation}
    \widetilde L_{t,j}=(Q_t^L)^\top(L_{t,j}+\lambda I_{d_1})Q_t^L,
    \qquad
    \widetilde R_{t,j}=(Q_t^R)^\top(R_{t,j}+\lambda I_{d_2})Q_t^R.
\end{equation}
Define the projected diagonal coordinates and their mean-eigenvalue variance estimate by
\begin{equation}
    \ell^L_{t,j,k}=[\widetilde L_{t,j}]_{kk},
    \qquad
    \ell^R_{t,j,k}=[\widetilde R_{t,j}]_{kk}.
\end{equation}
\begin{equation}
    \widehat{\Var}(\bar\lambda^L_{t,k})
    =
    \frac{1}{m(m-1)}\sum_{j=1}^m(\ell^L_{t,j,k}-\bar\ell^L_{t,k})^2,
    \qquad
    \bar\ell^L_{t,k}=\frac{1}{m}\sum_{j=1}^m\ell^L_{t,j,k},
\end{equation}
with the analogous definition for $\widehat{\Var}(\bar\lambda^R_{t,k})$.

For $f(x)=x^{-1/4}$, $f''(x)=\frac{5}{16}x^{-9/4}$. The corrected inverse-root eigenvalues are
\begin{equation}
    \widetilde d^L_{t,k}
    =
    (\lambda^L_{t,k})^{-1/4}
    -
    \frac{5}{32}(\lambda^L_{t,k})^{-9/4}\widehat{\Var}(\bar\lambda^L_{t,k}),
\end{equation}
\begin{equation}
    \widetilde d^R_{t,k}
    =
    (\lambda^R_{t,k})^{-1/4}
    -
    \frac{5}{32}(\lambda^R_{t,k})^{-9/4}\widehat{\Var}(\bar\lambda^R_{t,k}).
\end{equation}
For stability, corrected eigenvalues are clipped to the same range used by the Shampoo baseline:
\begin{equation}
    \widetilde d^L_{t,k}\leftarrow \min\{d_{\max},\max\{0,\widetilde d^L_{t,k}\}\},
    \qquad
    \widetilde d^R_{t,k}\leftarrow \min\{d_{\max},\max\{0,\widetilde d^R_{t,k}\}\}.
\end{equation}
The corrected left and right inverse-root matrices and the full corrected Shampoo update are
\begin{equation}
    \widetilde P_t^L
    =
    Q_t^L\mathrm{diag}(\widetilde d^L_{t,1},\ldots,\widetilde d^L_{t,d_1})(Q_t^L)^\top,
\end{equation}
\begin{equation}
    \widetilde P_t^R
    =
    Q_t^R\mathrm{diag}(\widetilde d^R_{t,1},\ldots,\widetilde d^R_{t,d_2})(Q_t^R)^\top.
\end{equation}
\begin{equation}
    U_t^{\mathrm{full}}
    =
    \widetilde P_t^L M_t \widetilde P_t^R.
\end{equation}
The same final Frobenius-norm clipping and decoupled weight decay are applied to standard and corrected Shampoo:
\begin{equation}
    U_t^{\mathrm{full}}
    \leftarrow
    U_t^{\mathrm{full}}\cdot \min\left\{1,\frac{c}{\|U_t^{\mathrm{full}}\|_F}\right\}.
\end{equation}
\begin{equation}
    W_t^{\mathrm{wd}}=W_t-\eta\lambda_{\mathrm{wd}}W_t,
    \qquad
    W_{t+1}=W_t^{\mathrm{wd}}-\eta U_t^{\mathrm{full}}.
\end{equation}
The Shampoo ablations are
\begin{align}
    U_t^{\mathrm{std}}
    &=
    (L_t+\lambda I)^{-1/4}M_t(R_t+\lambda I)^{-1/4}, \\
    U_t^{\mathrm{cf}}
    &=
    P_{B,t}^L M_{A,t}P_{B,t}^R, \\
    U_t^{\mathrm{inv}}
    &=
    \widetilde P_t^L M_t\widetilde P_t^R, \\
    U_t^{\mathrm{full}}
    &=
    \widetilde P_{B,t}^L M_{A,t}\widetilde P_{B,t}^R.
\end{align}

\section{Additional experimental results}
\label{app:additional-experiments}

This appendix records the targeted sweeps used to understand the bias-correction design. These are not exhaustive grids. They are diagnostic runs chosen to isolate learning-rate sensitivity, batch size, cross-fit mixing, inverse-correction timing, clipping, preconditioner recomputation, and denominator microbatch count. Unless otherwise stated, SFT results use the 500-example held-out set and lower is better.

\subsection{AdamW SFT sweeps}

The AdamW sweeps are split by the variable being tested. Tables~\ref{tab:app-adamw-controls} and~\ref{tab:app-adamw-full-lr} vary batch size and learning rate, respectively. Table~\ref{tab:app-adamw-impl} reports inverse-correction implementation choices. The strongest full-BC AdamW SFT run uses the larger b512 compute-matched setting with pre-EMA inverse correction and lr $1\mathrm{e}{-4}$.

\begin{table}[t]
\centering
\small
\setlength{\tabcolsep}{4pt}
\renewcommand{\arraystretch}{1.1}
\caption{AdamW SFT batch size sweep at lr $5\mathrm{e}{-5}$.}
\label{tab:app-adamw-controls}
\begin{tabular}{@{}lcc@{}}
\toprule
\rowcolor{headergray}
\textbf{Batch} & \textbf{Std} & \textbf{Full BC} \\
\midrule
b128 & 1.369 & \best{1.364} \\
b256 & 1.358 & \best{1.351} \\
b512 & 1.352 & \best{1.348} \\
\bottomrule
\end{tabular}
\end{table}

\begin{table}[t]
\centering
\small
\setlength{\tabcolsep}{4pt}
\renewcommand{\arraystretch}{1.1}
\caption{AdamW SFT learning-rate sweep at b512.}
\label{tab:app-adamw-full-lr}
\begin{tabular}{@{}lcc@{}}
\toprule
\rowcolor{headergray}
\textbf{LR} & \textbf{Std} & \textbf{Full BC} \\
\midrule
$2\mathrm{e}{-5}$ & \best{1.347} & 1.348 \\
$5\mathrm{e}{-5}$ & 1.352 & 1.348 \\
$1\mathrm{e}{-4}$ & 1.371 & \best{1.346} \\
$2\mathrm{e}{-4}$ & 1.415 & 1.395 \\
\bottomrule
\end{tabular}
\end{table}


\begin{table}[t]
\centering
\small
\setlength{\tabcolsep}{4pt}
\renewcommand{\arraystretch}{1.1}
\caption{AdamW inverse-correction implementation checks. Rows include denominator microbatch count, correction timing, and stability diagnostics.}
\label{tab:app-adamw-impl}
\begin{tabular}{@{}llcc@{}}
\toprule
\rowcolor{headergray}
\textbf{Variant} & \textbf{Control Variable} & \textbf{Setting} & \textbf{Eval} \\
\midrule
Inverse only & denominator microbatches & default & \best{1.341} \\
Inverse only & denominator microbatches & $m=8$ & 1.342 \\
Full BC & inverse timing & pre-EMA, b512, lr $1\mathrm{e}{-4}$ & 1.346 \\
\rowcolor{sectiongray}
\multicolumn{4}{@{}l}{\textit{Stability diagnostics}} \\
Full BC & warm start & 50 std steps then BC & 19.456${}^{\dagger}$ \\
Full BC & support clip & $\tau\in\{1,4,10\}$ & ---${}^{\ddagger}$ \\
\multicolumn{4}{@{}p{0.95\linewidth}}{\footnotesize ${}^{\dagger}$Training diverged after the warm-start handoff from standard to bias-corrected updates. ${}^{\ddagger}$Aggressive support clipping did not yield a usable final-evaluation checkpoint.} \\
\bottomrule
\end{tabular}
\end{table}

\subsection{Cross-fit mixing sweeps}

Tables~\ref{tab:app-alpha-b128} and~\ref{tab:app-alpha-b512} report AdamW cross-fit mixing experiments. Here $\alpha$ controls how strongly the denominator is mixed between same-batch and cross-fitted estimates. Keeping batch size fixed makes the pattern clearer: small rolling-$B$ mixing recovers much of the standard behavior at b128, while fixed full decoupling at b512 needs a larger learning rate.

\begin{table}[t]
\centering
\small
\setlength{\tabcolsep}{4pt}
\renewcommand{\arraystretch}{1.1}
\caption{AdamW SFT cross-fit mixing at b128 and lr $2\mathrm{e}{-5}$.}
\label{tab:app-alpha-b128}
\begin{tabular}{@{}llc@{}}
\toprule
\rowcolor{headergray}
\textbf{Variant} & \textbf{Mixing Coefficient} & \textbf{Eval} \\
\midrule
Cross-fit & $\alpha=0$ & 1.349 \\
Cross-fit & $\alpha=0.25$ & 1.349 \\
Cross-fit & adaptive $\alpha$ & 1.349 \\
Rolling cross-fit & $\alpha_{\max}=0.1$ & \best{1.342} \\
Rolling cross-fit & $\alpha_{\max}=0.25$ & 1.343 \\
Rolling cross-fit & $\alpha_{\max}=1.0$ & 1.345 \\
\bottomrule
\end{tabular}
\end{table}

\begin{table}[t]
\centering
\small
\setlength{\tabcolsep}{4pt}
\renewcommand{\arraystretch}{1.1}
\caption{AdamW fixed cross-fit at b512. The related controls are mixing strength and learning rate.}
\label{tab:app-alpha-b512}
\begin{tabular}{@{}lccc@{}}
\toprule
\rowcolor{headergray}
\textbf{Variant} & \textbf{Mixing Coefficient} & \textbf{LR} & \textbf{Eval} \\
\midrule
Fixed cross-fit & $\alpha=0.5$ & $2\mathrm{e}{-5}$ & 1.353 \\
Fixed cross-fit & $\alpha=1.0$ & $2\mathrm{e}{-5}$ & 1.356 \\
Fixed cross-fit & $\alpha=1.0$ & $1\mathrm{e}{-4}$ & \best{1.345} \\
\bottomrule
\end{tabular}
\end{table}

\subsection{Sophia SFT sweeps}

Table~\ref{tab:app-sophia-sweeps} reports the Sophia SFT sweeps. Sophia SFT is dominated by its clipped-ratio update: the default b512 full-BC runs are stable and nearly tied with standard, while the larger 5K evaluation slightly favors full BC. Aggressive cross-fitting at lr $1\mathrm{e}{-4}$ is a high-learning-rate failure case.

\begin{table}[t]
\centering
\small
\setlength{\tabcolsep}{4pt}
\renewcommand{\arraystretch}{1.1}
\caption{Sophia SFT sweeps.}
\label{tab:app-sophia-sweeps}
\begin{tabular}{@{}llllr@{}}
\toprule
\rowcolor{headergray}
\textbf{Batch/LR} & \textbf{Variant} & \textbf{Eval Set} & \textbf{Std} & \textbf{Corrected} \\
\midrule
\rowcolor{sectiongray}
\multicolumn{5}{@{}l}{\textit{Full-BC timing at fixed b512 and lr $2\mathrm{e}{-5}$}} \\
b512 / $2\mathrm{e}{-5}$ & post-EMA full BC & 500 & 1.3892 & \best{1.3892} \\
b512 / $2\mathrm{e}{-5}$ & pre-EMA full BC & 500 & \best{1.3892} & 1.3893 \\
b512 / $2\mathrm{e}{-5}$ & full BC & 5K & 1.3422 & \best{1.3418} \\
\addlinespace[2pt]
\rowcolor{sectiongray}
\multicolumn{5}{@{}l}{\textit{Cross-fit-only learning-rate check}} \\
b512 / $2\mathrm{e}{-5}$ & cross-fit only & 500 & \best{1.3892} & 1.3899 \\
b512 / $1\mathrm{e}{-4}$ & cross-fit only & 500 & \best{1.3892} & 1.6570 \\
\addlinespace[2pt]
\rowcolor{sectiongray}
\multicolumn{5}{@{}l}{\textit{Low-batch stress test}} \\
b128 / $2\mathrm{e}{-5}$ & full-BC recipe & 500 & \best{1.4841} & 1.5699 \\
\bottomrule
\end{tabular}
\end{table}

\subsection{Shampoo SFT sweeps}

Table~\ref{tab:app-shampoo-sweeps} reports Shampoo layer-coverage and inverse-root recomputation sweeps. Attention-only Shampoo is the most stable setting; routing MLP matrices through Shampoo increases coverage but worsens the short SFT loss for both standard and corrected variants in these experiments. These rows are grouped separately from Sophia because they answer a different question: how much structured preconditioner coverage and eigendecomposition frequency are needed in short SFT. Across all rows in Table~\ref{tab:app-shampoo-sweeps}, the Std and Full BC eval losses agree to within $\sim 0.005$, which is at or below the 500-example single-seed noise floor; the table documents that the Shampoo correction does not measurably help in short SFT, and Section~\ref{sec:limitations} cites this as one of the regimes where the variance cost of sample splitting can offset the bias benefit.

\begin{table}[t]
\centering
\small
\setlength{\tabcolsep}{4pt}
\renewcommand{\arraystretch}{1.1}
\caption{Shampoo SFT sweeps over layer coverage and inverse-root recomputation.}
\label{tab:app-shampoo-sweeps}
\begin{tabular}{@{}llllrr@{}}
\toprule
\rowcolor{headergray}
\textbf{Coverage} & \textbf{Max Dim} & \textbf{Root / Momentum Setting} & \textbf{Eval} & \textbf{Std} & \textbf{Full BC} \\
\midrule
\rowcolor{sectiongray}
\multicolumn{6}{@{}l}{\textit{Attention-only Shampoo}} \\
Attention only & 2048 & root freq. 10 & 500 & \best{1.347} & 1.347 \\
Attention only & 2048 & root freq. 10 & 5K & \best{1.298} & 1.298 \\
Attention only, b128 & 2048 & root freq. 10 & 500 & \best{1.342} & 1.343 \\
\addlinespace[2pt]
\rowcolor{sectiongray}
\multicolumn{6}{@{}l}{\textit{Attention + MLP Shampoo}} \\
Attention + MLP & 4864 & root freq. 10 & 500 & \best{1.381} & 1.385 \\
Attention + MLP & 4864 & root freq. 2 & 500 & \best{1.386} & 1.388 \\
Attention + MLP & 4864 & root freq. 2, $\beta=0.5$ & 500 & \best{1.387} & 1.389 \\
Attention + MLP & 5000 & literature-style params & 500 & \best{1.363} & 1.368 \\
\bottomrule
\end{tabular}
\end{table}

\subsection{Additional pretraining comparisons}

Table~\ref{tab:app-pretrain-comparisons} reports the additional pretraining comparisons used to contextualize the main results.

\begin{table}[t]
\centering
\small
\setlength{\tabcolsep}{4pt}
\renewcommand{\arraystretch}{1.1}
\caption{Additional pretraining comparisons.}
\label{tab:app-pretrain-comparisons}
\begin{tabular}{@{}lllrr@{}}
\toprule
\rowcolor{headergray}
\textbf{Optimizer} & \textbf{Variant} & \textbf{LR} & \textbf{Final Train} & \textbf{Eval} \\
\midrule
\rowcolor{sectiongray}
\multicolumn{5}{@{}l}{\textit{AdamW}} \\
AdamW & Standard, clean & $6\mathrm{e}{-4}$ & 4.832 & 4.836 \\
AdamW & LOO+Jensen, clean & emb $6\mathrm{e}{-4}$ / dense $9\mathrm{e}{-4}$ & \best{4.733} & \best{4.687} \\
AdamW & Standard, noisy train & $6\mathrm{e}{-4}$ & \best{4.853} & 4.823 \\
AdamW & LOO+Jensen, noisy train & emb $6\mathrm{e}{-4}$ / dense $9\mathrm{e}{-4}$ & 4.872 & \best{4.803} \\
AdamW & Inverse only, clean & $6\mathrm{e}{-4}$ & 4.826 & 4.830 \\
AdamW & Two-fold full BC, clean & $6\mathrm{e}{-4}$ & 5.281 & 5.285 \\
\addlinespace[2pt]
\rowcolor{sectiongray}
\multicolumn{5}{@{}l}{\textit{Sophia-G}} \\
Sophia & Standard & $6\mathrm{e}{-4}$ & 6.663 & 6.665 \\
Sophia & Full BC & $6\mathrm{e}{-4}$ & \best{6.592} & \best{6.595} \\
\addlinespace[2pt]
\rowcolor{sectiongray}
\multicolumn{5}{@{}l}{\textit{Shampoo}} \\
Shampoo & Standard & $6\mathrm{e}{-4}$ & 5.787 & 5.792 \\
Shampoo & Full BC & $1\mathrm{e}{-3}$ & \best{5.676} & \best{5.681} \\
\bottomrule
\end{tabular}
\end{table}

\clearpage

\section{Code, Assets, and LLM Usage}
\label{sec:reproducibility-assets-llm}

\paragraph{Code availability.}
The code for the experiments is available at \url{https://github.com/fastino-ai/preconditioner-bias-correction}. The repository includes optimizer implementations, training scripts, evaluation scripts, plotting utilities, and the data-preparation script for packed FineWeb-Edu sequences. It is intended to reproduce the runs summarized in Tables~\ref{tab:pretraining-results}--\ref{tab:sft-main} and the additional sweeps in Appendix~\ref{app:additional-experiments}, subject to the same hardware and data-access assumptions listed in Table~\ref{tab:training-config}.

\paragraph{Existing assets and licenses.}
The experiments use the Qwen/Qwen2.5-0.5B model and tokenizer, which are released under Apache 2.0; the HuggingFaceFW/fineweb-edu dataset, which is released under the Open Data Commons Attribution License (ODC-By) v1.0; and the yahma/alpaca-cleaned instruction dataset, which is released under CC-BY-4.0. The optimizer baselines are credited to their original papers in Section~\ref{sec:related-work}. We do not release new model checkpoints or new datasets.

\paragraph{LLM usage.}
Large language models are the optimization targets in this work. LLM-based assistants were used for drafting, editing manuscript text, understanding technical concepts, writing code, and helping facilitate or run experiments. They were not used to generate training data, labels, mathematical results, or experimental measurements; all claims, derivations, code, and reported numbers were checked by the authors against the source code and run artifacts.

\end{document}